\documentclass[11pt, journal, onecolumn]{IEEEtran}

\usepackage{amsmath,amssymb,graphicx,paralist,subfigure,pdfpages,cite,algorithm,algpseudocode,paralist}
\usepackage{amsthm}
\usepackage{verbatim}
\usepackage{algorithm}
\usepackage{algpseudocode}
\newtheorem{lem}{Lemma} 
\newtheorem{theorem}{Theorem}

\newtheorem{cor}{Corollary}

\newtheorem{assump}{Assumption}

\def\ln{{\rm ln}}

\def\R{\mathbb{R}}
\def\mc{\mathcal}
\def\mb{\mathbf}
\def\mbb{\mathbb}
\def\ra{\rightarrow}
\def\P{\mathbf{P}}

\def\bpi{\boldsymbol\pi}

\IEEEoverridecommandlockouts
\def\SFO{{\texttt{SFO}}}
\def\GP{\textbf{\texttt{GP}}}

\def\ADDOPT{\textbf{\texttt{ADDOPT}}}
\def\SGP{\textbf{\texttt{SGP}}}
\def\SA{\textbf{\texttt{S-ADDOPT}}}
\def\CGD{\textbf{\texttt{GD}}}
\def\CSGD{\textbf{\texttt{SGD}}}
\def\DGD{\textbf{\texttt{DGD}}}
\def\DSGD{\textbf{\texttt{DSGD}}}

\renewcommand{\baselinestretch}{1.5}

\def\mt{\times}
\def\wh{\widehat}
\def\mt{\mathbf{x}}
\def\mbb{\mathbb}
\def\mb{\mathbf}
\def\mc{\mathcal}
\def\wh{\widehat}
\def\wt{\widetilde}
\def\ol{\overline}
\def\ul{\underline}

\def\bds{\boldsymbol}
\newcommand{\mn}[1]{{\left\vert\kern-0.25ex\left\vert\kern-0.25ex\left\vert\kern0.3ex #1 
		\kern0.3ex\right\vert\kern-0.25ex\right\vert\kern-0.25ex\right\vert}}

\begin{document}
\title{\huge \SA:~Decentralized stochastic first-order\\ optimization over directed graphs}
\author{
Muhammad I. Qureshi$^\dagger$, Ran Xin$^\ddagger$, Soummya Kar$^\ddagger$, and Usman A. Khan$^\dagger$\\
$^\dagger$Tufts University, Medford, MA, USA, $^\ddagger$Carnegie Mellon University, Pittsburgh, PA, USA
\thanks{The authors acknowledge the support of NSF  under awards  CCF-1513936, CMMI-1903972, and CBET-1935555.}
}
\maketitle
	
\begin{abstract}
In this report, we study decentralized stochastic optimization to minimize a sum of smooth and strongly convex cost functions when the functions are distributed over a directed network of nodes. In contrast to the existing work, we use gradient tracking to improve certain aspects of the resulting algorithm. In particular, we propose the~$\SA$ algorithm that assumes a stochastic first-order oracle at each node and show that for a constant step-size~$\alpha$, each node converges linearly inside an error ball around the optimal solution, the size of which is controlled by~$\alpha$. For decaying step-sizes~$\mc{O}(1/k)$, we show that~$\SA$ reaches the exact solution sublinearly at~$\mc{O}(1/k)$ and its convergence is asymptotically network-independent. Thus the asymptotic behavior of~$\SA$ is comparable to the centralized stochastic gradient descent. Numerical experiments over both strongly convex and non-convex problems illustrate the convergence behavior and the  performance comparison of the proposed algorithm.
\end{abstract}

\section{Introduction}
This report considers minimizing a sum of smooth and strongly convex functions~$F(\mb{z}) \!=\! \sum_{i=1}^n f_i(\mb{z})$ over a network of $n$ nodes. We assume that each~$f_i$ is private~to only on node~$i$ and that the nodes communicate over a directed graph (digraph) to solve the underlying problem. Such problems have found significant applications traditionally in the areas of signal processing and control~\cite{Optimization_Zavlanos,sam_proc:17} and more recently in machine learning problems~\cite{DSVM_Forero, BigData_Bajwa,YANG2019278,bottou2018optimization}. Gradient descent (\CGD) is one of the simplest algorithms for function minimization and requires the true gradient~$\nabla F$. When this information is not available,~\CGD~is implemented with stochastic gradients and the resulting method is called stochastic gradient descent~(\CSGD). As the data becomes large-scale and geographically diverse, \CGD~and~\CSGD~present storage and  communication challenges. In such cases, decentralized methods are attractive as they are locally implemented and rely on communication among nearby nodes. 

Related work on decentralized first-order methods can be found in~\cite{DGD_nedich,DGD_Yuan,DSGD_nedich,Diffusion_Chen,GT_SAGA,DSGT_Pu}. Of relevance is Distributed Gradient Descent (\DGD) that converges sublinearly to the optimal solution with decaying step-sizes\cite{DGD_nedich} and linearly to an inexact solution with a constant step-size~\cite{DGD_Yuan}. Its stochastic variant~\DSGD~can be found in~\cite{DSGD_nedich,Diffusion_Chen}, which is further extended with the help of gradient tracking~\cite{GT_CDC,add-opt,AB} in~\cite{DSGT_Pu} where inexact linear convergence in addition to asymptotic network independence are shown; see also~\cite{SGP_Olshevsky,pu2019sharp,SP_SPMag_2020} and references therein. More recently, variance reduction has been used to show linear convergence for smooth and strongly convex finite-sum problems~\cite{GT_SAGA}. However, all of these decentralized stochastic algorithms are built on undirected graphs, see~\cite{GT_SAGA_SPM} for a friendly tutorial. Related work on directed graphs includes~\cite{opdirect_Tsianous,GP_neidch,DEXTRA,add-opt,xi_tac4:17,FROST,AB} where true gradients are used, and~\cite{SGP_nedich,SGP_ICML,SGP_Olshevsky,DSGT_Xin} on stochastic methods, all of which use the push-sum algorithm~\cite{ac_directed0} to achieve agreement with an exception  of~\cite{xi_tac1:16,xi_neuro:17,AB,DSGT_Xin} that employ updates with both row and column stochastic weights to avoid the eigenvector estimation in push-sum. 

In this report, we present~\SA~for decentralized stochastic optimization over directed graphs. In particular, \SA~adds gradient tracking to~\SGP~(stochastic gradient push)~\cite{SGP_nedich,SGP_ICML,SGP_Olshevsky} and can be viewed as a stochastic extension of~\ADDOPT~\cite{add-opt,DIGing} that uses true gradients. Of significant relevance is~\cite{DSGT_Pu} that is applicable to undirected graphs and is based on doubly stochastic weights. Since~\SA~is based on directed graphs, it essentially extends the algorithm in~\cite{DSGT_Pu} with the help of push-sum when the network weights are restricted to be column stochastic. A similar algorithm based on row-stochastic weights is also immediate by apply the extension and analysis in this report to FROST~\cite{xi_tac4:17,FROST}.

The main contributions of this report are as follows:
\begin{inparaenum}[(i)]
\item We develop a stochastic algorithm over directed graphs by combining push-sum with gradient tacking;
\item For a constant step-size~$\alpha$, we show that each node converges linearly inside an error ball around the optimal solution, and further show that the size of the error ball is controlled by~$\alpha$.
\item For decaying step-sizes~$\mc{O}(1/k)$, we show that~$\SA$ is asymptotically network-independent and reaches the exact solution sublinearly at~$\mc{O}(1/k)$, while the network agreement error decays at a faster rate of~$\mc O(1/k^2)$.
\item We explicitly quantify the directed nature of graphs using a directivity constant~$\tau$, which makes this work a generalization of~\DSGD,~\SGP,~and the method proposed in~\cite{DSGT_Pu}. The directivity constant~$\tau$ is~$1$ for undirected graphs and thus the results apply to undirected graphs as a special case.
\end{inparaenum}
The rest of this report is organized as follows. We formalize the optimization problem, list the underlying assumptions, and describe~$\SA$~in Section~\ref{sec_pf}. We then present the main results in Section~\ref{sec_mr}~and the convergence analysis in Section~\ref{sec_pt}. Finally, we provide numerical experiments in Section~\ref{sec_ne}~and conclude the report in Section~\ref{sec_conc}.

\textbf{Basic Notation:} We use uppercase italic letters for matrices and lowercase bold letters for vectors. We use~$I_n$~for the~$n \times n$ identity matrix and~$\mb{1}_n$ denotes the column vector of~$n$ ones. A column stochastic matrix is such that it is non-negative and all of its columns sum to~$1$. For a primitive column stochastic matrix~$\ul{B} \in \mbb{R}^{n \times n}$, we have~$\ul{B}^\infty \!=\! \bds{\pi} \mb{1}_n^\top$, from the Perron-Frobenius theorem~\cite{hornjohnson}, where~$\bds{\pi}$ and~$\mb{1}^\top_n$ are its right and left Perron eigenvectors. For a matrix~$G$,~$\rho(G)$ is its spectral radius. We denote the Euclidean (vector) norm by~$\|\!\cdot\!\|_2$ and define a weighted inner product as~$\langle\mb{x}, \mb{y}\rangle_{\bds{\pi}} \!:=\! \mb{x}^\top \mbox{diag}(\bds{\pi})^{-1} \mb{y}$, for~$\mb{x},\mb{y}\in$~$\mbb{R}^p$, which leads to a weighted Euclidean norm:~$\|\mb{x}\|_{\bds{\pi}} := \|\mbox{diag}(\sqrt{\bds{\pi}})^{-1} \mb{x}\|_2$. We denote~$\mn{\cdot}_{\bds{\pi}}$ as the matrix norm induced by~$\|\!\cdot\!\|_{\bds{\pi}}$ such that~$\forall X \in \mbb{R}^{n \times n}$,~$\mn{X} := \mn{\mbox{diag}(\sqrt{\bds{\pi}})^{-1}~X~\mbox{diag}(\sqrt{\bds{\pi}})}_2$. Note that these norms are related as~$\|\cdot\|_{\bds{\pi}} \leq \ul{{\pi}}^{-0.5}\|\cdot\|_2$ and~$\|\cdot\|_2 \leq \ol{{\pi}}^{0.5}\|\cdot\|_{\bds{\pi}}$, where~$\ol{\pi}$ and~$\ul{\pi}$ are the maximum and minimum elements in~$\bpi$, while~$\mn{\ul{B}}_{\bds{\pi}} \!=\! \mn{\ul{B}^\infty}_{\bds{\pi}} \!=\! \mn{I_n - \ul{B}^\infty}_{\bds{\pi}} \!=\! 1$. Finally, it is shown in~\cite{DSGT_Xin}  that~$\sigma_B:=\mn{\ul{B}-\ul{B}^\infty}_{\bpi}<1$.

\section{Problem Formulation} \label{sec_pf}
Consider~$n$ nodes communicating over a strongly-connected directed graph (digraph), $\mc G=(\mc{V},\mc{E})$, where~$\mc{V} = \{1,2,3,\dots,n\}$ is the set of agents and~$\mc{E}$ is the collection of ordered pairs, $(i,j),i,j\in \mc{V}$, such that node $i$ receives information from node $j$. We let~$\mc{N}_i^{\mbox{\scriptsize{out}}}$ (resp.~$\mc{N}_i^{\mbox{\scriptsize{in}}}$) to denote the set of out-neighbors (resp. in-neighbors) of node~$i$, i.e., nodes that can receive information from~$i$, and~$|\mc{N}_i^{\mbox{\scriptsize{out}}}|$ is the out-degree of node~$i$. Note that both~$\mc{N}_i^{\mbox{\scriptsize{out}}}$ and~$\mc{N}_i^{\mbox{\scriptsize{in}}}$ include node~$i$. The nodes collaborate to solve the following optimization  problem:
\vspace{-0.2cm}
\[
\P:\qquad \min_{\mb z\in\mbb R^p} F(\mb{z}) := \frac{1}{n} \sum_{i=1}^{n} f_i(\mb{z}),
\]
where each node~$i$ possesses a private cost function~$f_i: \mbb{R}^p \rightarrow \mbb{R}$. We make the following assumptions.

{
\assump \label{comm_graph} The communication graph~$\mc G$ is a strongly-connected directed graph and each node has the knowledge of its out-degree~$|\mc{N}_i^{\mbox{\scriptsize{out}}}|$.

\assump \label{smooth_convex} Each local cost function~$f_i$ (and thus~$F$) is~$\mu$-strongly convex and~$\ell$-smooth, i.e.,~$\forall \mb x,\mb y \in \mbb{R}^p$ and $\forall i \in \mc{V},$ there exist positive constants~$\mu$ and~$\ell$ such that 
\[\frac{\mu}{2}\|\mb{x}-\mb{y}\|_2^2\leq f_i(\mb{y})-f_i(\mb{x})-\nabla f_i(\mb{x})^\top(\mb{y}-\mb{x})\leq\frac{\ell}{2}\|\mb{x}-\mb{y}\|_2^2.\]
Note that the ratio~$\kappa:=\frac{\ell}{\mu}$ is called the condition number of the function~$f_i$. We have that~$\ell\geq \mu$ and thus~$\kappa\geq1$.

\assump \label{SFO_assump} Each node has access to a stochastic first-order oracle~\SFO~that returns a stochastic gradient~$\nabla \widehat{f}_i(\mb {z}_{k}^i)$ for any~$\mb z_k^i\in\mbb{R}^p$ such that 
\begin{align*}
&\mbb{E} \left[{ \nabla \widehat{f}_i(\mb z_k^i)}|\mb z_k^i\right] = \nabla f_i(\mb z_k^i),\\
&\mbb{E} \left[ \|{ \nabla \widehat{f}_i(\mb {z}_k^i)} - \nabla f_i( \mb{z}_k^i)\|_2^2|\mb {z}_k^i \right] \leq \sigma^2.
\end{align*}}

\vspace{-0.3cm}
\noindent These assumptions are standard in the related literature. The bounded variance assumption however can be relaxed, see~\cite{bottou2018optimization}, for example. Due to Assumption~\ref{smooth_convex}, we note that~$F$ has a unique minimizer that is denoted by~$\mb z^*$. The proposed algorithm to solve Problem~$\P$~is described~next. 

\subsection{$\SA$: Algorithm}
The~$\SA$ algorithm to solve Problem~$\P$ is formally described in Algorithm~\ref{sa_alg}. We note that the set of weights~$\ul{B}=\{b_{ij}\}$ is such that~$\ul{B}$ is column stochastic. A valid choice is~$b_{ji}=|\mc N_i^{\mbox{\scriptsize out}}|^{-1}$, for each~$j\in\mc N_i^{\mbox{\scriptsize out}}$ and zero otherwise, recall Assumption~\ref{comm_graph}. Each agent~$i$ maintains three state vectors, i.e.,~$\mb x_k^i,\mb w_k^i, \mb z_k^i \in \mbb{R}^p$ and a (positive) scalar~$y_k^i$ at each iteration~$k$. The first update~$\mb x_{k+1}^i$ is similar to~\DSGD, where the stochastic gradient~$\nabla \wh{f}_i(\mb x_k^i)$ is replaced with~$\mb w_k^i$. This auxiliary variable~$\mb w_k^i$ is based on dynamic average-consensus~\cite{DAC} and in fact tracks the global gradient~$\nabla F$ when viewed as a non-stochastic update (see~\cite{GT_CDC,di2016next,add-opt,AB} for details). However, since the weight matrix~$\ul B$ is not row-stochastic, the variables~$\mb x_k^i$'s do not agree on a solution and converge with a certain imbalance that is due to the fact that~$\mb 1_n$ is not the right Perron eigenvector of~$\ul B$. This imbalance is canceled in the~$\mb z_k^i$-update with the help of a scaling by~$y_k^i$, since~$y_k^i$ estimates the~$i$-th component of~$\bpi$ (recall that~$\ul B \bpi = \bpi$). We note that~\SA~is in fact a stochastic extension of~\ADDOPT, where true local gradients~$\nabla f_i$'s are used at each node.
\begin{algorithm}[!h]
\caption{\SA: At each node~$i$}
\label{sa_alg}
	\begin{algorithmic}[1]
		\Require ${\mb x_0^i\in\mbb R^p},{\mb z_0^i=\mb x_0^i}, {y_0^i=1},{\mb{w}_0^i=\nabla \wh f_i(\mb z_0^i)},\alpha>0$
		\For{$k= 0,1,2,\cdots$}
	    \State \textbf{State update:} $\mb{x}_{k+1}^i = \sum_{j=1}^{n}b_{ij} \mb {x}_k^j - \alpha \mb {w}_k^i$
	    \State \textbf{Eigenvector est.:} $y_{k+1}^i = \sum_{j=1}^{n}b_{ij} {y}_k^j$
	    \State \textbf{Push-sum update:} $\mb{z}_{k+1}^i = {\mb{x}_{k+1}^i}/{ {y}_{k+1}^i}$
	    \State \textbf{Gradient tracking update:} $\mb {w}_{k+1}^i = \sum_{j=1}^{n}b_{ij} \mb {w}_k^j + \nabla \widehat{f}_i(\mb {z}_{k+1}^i) - \nabla \widehat{f}_i(\mb {z}_{k}^i)$
		\EndFor
	\end{algorithmic}
\end{algorithm}

\SA~can be compactly written in a vector form with the help of the following notation. Let~$\mb x_k, \mb z_k, \mb w_k$, all in~$\mbb R^{np}$ concatenate the local states~$\mb x_k^i, \mb z_k^i, \mb w_k^i$ (all in~$\mbb{R}^p$) at the nodes and~$\mb y_k\in\mbb R^n$ stacks the~$y_k^i$'s. Let~$\otimes$ denote the Kronecker product and define~$B := \underline{B}\otimes I_p$, and let~$Y_{k} := \mbox{diag} (\mb{y}_k) \otimes I_p$. Then~\SA~described in Algorithm~\ref{sa_alg} can be written  in a vector form as
\begin{subequations}\label{SADDOPTv}
\begin{align}
\mb {x}_{k+1} &= B \mb {x}_k - \alpha \mb {w}_k,\label{SADDOPTv1}\\
\mb {y}_{k+1} &= \ul{B} \mb {y}_{k},\label{SADDOPTv2} \\
\mb {z}_{k+1} &= Y_{k+1}^{-1} \mb {x}_{k+1},\label{SADDOPTv3}\\
\mb {w}_{k+1} &= B \mb {w}_k + \nabla \widehat{f}(\mb {z}_{k+1}) - \nabla \widehat{f}(\mb {z}_{k}).\label{SADDOPTv4}
\end{align}
\end{subequations}
In the following sections, we summarize the main results (Section~\ref{sec_mr}) and provide the convergence analysis (Section~\ref{sec_pt}) of~\SA. Subsequently, we compare its performance with related algorithms on digraphs in Section~\ref{sec_ne}.

\section{Main Results} \label{sec_mr}
We use~$p=1$ for simplicity and thus~${B = \ul B}$. Before we proceed, we define~${\ol{\mb x}_k:=\frac{1}{n}\mb 1_n^\top \mb x_k}$, and~${\wh{\mb x}_k:=B^\infty \mb x_k}$, which are the mean and weighted averages of~$\mb x_k^i$'s, respectively, and~${y := \sup_k \mn{Y_k}_2}$,~${y_- := \sup_k \mn{Y_k^{-1}}_2}.$
We next provide two useful lemmas.
\begin{lem} \label{lem1}\cite{add-opt,DSGT_Xin}
Consider Assumption~$\ref{comm_graph}$~and~define ${Y^\infty:=\lim_{k\ra\infty} Y_k}$,~${h := \ol{\pi} / \ul{\pi}}$, and ~${\beta := \sqrt{h} \|\mb{1}_n - n \bds{\pi} \|_2}$. Then ${\mn{Y_{k} - Y^\infty}_2 \leq \beta \sigma_B^k}$, $\forall k \geq 0$.
\end{lem}
\begin{proof} Note that~$\forall k \geq 0,~\mb{y}_\infty = B^\infty \mb{y}_{k}$. Thus we have
\begin{align*}
    \mn{Y_k - Y^\infty}_2 &\leq \| \mb{y}_k - \mb{y}_\infty \|_2 \leq \sqrt{\ol{\pi}} \mn{ B - B^\infty }_{\bpi} \| \mb{y}_{k-1} - \mb{y}_\infty \|_{\bpi} \leq \sigma_B^k \sqrt{h} \| \mb{y}_{0} - \mb{y}_\infty \|_{2}. 
\end{align*}
and the proof follows.
\end{proof}
\begin{lem} \label{lem_e2}
Define~${\mb e_k:= \frac{1}{n}\mbb{E}[\|\mb{z}_k - \mb1_n\mb{z}^* \|_2 ^2]}$ as the mean error in the network. We have
\vspace{-0.2cm}
\begin{align} \label{e1}
\mb e_k &\leq \frac{\omega}{n} \mbb{E} \|\mb x_k - \widehat{\mb x}_k \|_{\bpi}^2 + \omega \beta^2 \! \sigma_B^{k} \|\mb z^* \|_2^2 + \omega y^2 \mbb{E} \|\ol{\mb x}_k - \mb z^* \|^2_{\bpi}, \\ \label{e2}
\mb e_k &\leq \psi\mbb{E}\|\mb{x}_k - \widehat{\mb{x}}_k \|^2_{\bpi} + \psi \beta \sigma_B^k \mbb{E}\|\mb{x}_k\|_{\bpi}^2 + 2 \mbb{E}\|\ol{\mb{x}}_k - \mb{z}^* \|^2_2,
\end{align}
where~$\omega:=3y_-^2\ol{\pi}$ and~${\psi:={2 y_-^2 \overline{\pi} (1 + \beta) }/{n}}$.
\end{lem}

We now provide the main results on~\SA. 

\begin{theorem}\label{th1}
Let Assumptions~\ref{comm_graph},~\ref{smooth_convex}, and~\ref{SFO_assump} hold and let the step-size~$\alpha$ be a constant such that,
\begin{align} \label{alphaBound}
\alpha &\leq \frac{1}{\ell\sqrt\kappa}\cdot\frac{(1-\sigma_B^2)^2}{51\sqrt{\tau}},
\end{align}
where~$\tau := y_-^6 y^2 h(1+\beta)$~is the directivity constant.
Then~$\mb e_k$ converges linearly, at a rate~$\gamma,\gamma\in[0,1)$, to a ball around~$\mb z^*$, i.e.,
\vspace{-0.05cm}
\begin{eqnarray} \label{t1_eq}
\limsup_{k \rightarrow \infty} \mb{e}_k =  \alpha\:\mc{O} \left(\frac{\sigma^2 }{n \mu} \right)
+ \alpha^2\:\mc{O} \left(\frac{\ell^2 \sigma^2 }{\mu^2(1-\sigma_B^2)^4}\right).
\end{eqnarray}
\end{theorem}

The proof of Theorem~\ref{th1} is provided in the next Section. It essentially shows that~\SA~converges linearly with a constant step-size to an error ball around~$\mb{z}^*$, the size of which however is controlled by~$\alpha$. We note that~${\tau \geq 1}$~can be considered as a directivity constant and is large when the graph is more directed as quantified by e.g., the constant~$h$ (in addition to the other constants in~$\tau$); clearly, for undirected graphs~$\tau = 1$ and thus Theorem~\ref{th1} is applicable to undirected graphs as a special case. We further note that the first term in~\eqref{t1_eq} is due to the variance~$\sigma^2$ of the stochastic gradients and does not have a network dependence, i.e., a scaling with~$(1-\sigma_B^2)^{-1}$. The rate of convergence of $\SA$ thus is comparable to the~\CSGD~(up to some constant factors) when the step-size~$\alpha$ is sufficiently small since the second term has a higher order of~$\alpha$. The result in Theorem~\ref{th1} is similar to what was obtained for undirected graphs in~\cite{DSGT_Pu}, where the network dependence is~$\mc O((1-\sigma_B^2)^{-3})$. We next provide an upper bound on the linear rate~$\gamma$.

\begin{cor} \label{c1}
Let Assumptions~\ref{comm_graph},~\ref{smooth_convex} and~\ref{SFO_assump} hold. If the step-size follows~$\alpha \leq  \frac{3}{40} \left(\frac{1 - \sigma_B^2}{\mu}\right),$
then the linear rate parameter~$\gamma$ in Theorem~\ref{th1} is such that
\[
\gamma \leq 1 - \frac{\alpha \mu}{3}.
\]
\end{cor}

The proof of Corollary~\ref{c1} is available in Appendix~\ref{App2} and follows the same arguments as in~\cite{DSGT_Pu}. Going back to Theorem~\ref{th1}, note that the exact expression of~\eqref{t1_eq} is provided later in the convergence analysis, see~\eqref{ztozstar2}, where we dropped the higher powers of~$\alpha$ when writing~\eqref{t1_eq}. We note from~\eqref{ztozstar2} that all terms in the residual are a function of~$\sigma^2$ and thus~\SA~recovers the exact linear convergence as~$\sigma^2$ vanishes. When~$\sigma^2$ is not zero, exact convergence is achievable albeit at a sublinear rate with decaying step-sizes. We provide this result below. 
\begin{theorem} \label{th2} Let Assumptions~\ref{comm_graph},~\ref{smooth_convex}, and~\ref{SFO_assump} hold. Consider $\SA$ with decaying step-sizes~$\alpha_k := \frac{\theta}{m+k},\theta > \frac{1}{\mu}$ and~$m$ such that
{
\begin{align*}
\left\{
\begin{array}{l}
m > \max \left \{ \frac{\theta(\ell + \mu)}{2}, \frac{6 \ell \theta y_- \sqrt{(1+\sigma_B^2)h}} {1-\sigma_B^2} \right \}, \\
\frac{(1-\sigma_B^2)^2}{6 \theta^2(1+\sigma_B^2)} \left(\frac{1-\sigma_B^2}{2} - \frac{2 m + 1}{(m+1)^2} \right) > \frac{E_2}{m^2} + \left(\frac{\theta^3 \ell^6 E_1 E_3}{m^4 n \left(\theta \mu - 1 \right) } \right) \left(\frac{\theta\mu+m}{m\mu} \right),
\end{array}
\right. 
\end{align*}}for some constants~$E_1,E_2,E_3.$ Select~$\wt{S}$~large enough such that~$\forall k \geq \wt{S}, \sigma_B^k \leq 
\frac{1}{n(m+k)^2}$, then we have
\begin{align*}
\mbb{E} \|\mb x_k - \mb{\widehat{x}}_k \|_{\bpi}^2 &\leq \frac{\mc{O}(1)}{(m+k)^{2}},\\
\mbb{E} \|\mb{\ol{x}}_k - \mb{z}^* \|_{2}^2 &\leq \frac{2 \theta^2 \sigma^2}{n(\theta \mu - 1)(m+k)} + \frac{\mc{O}(1)}{(m+k)^{\theta \mu}} + \frac{\mc{O}(1)}{(m+k)^{2}},
\end{align*}
which leads to~${\mb e_k\ra0}$ at a network-independent convergence rate of~$\mc O(\frac{1}{k})$.
\end{theorem}

Theorem~\ref{th2}, formally analyzed in the next section, shows that the error~$\mb e_k$ in~$\SA$~ asymptotically converges to the exact solution at a rate dominant by~$\frac{4 \theta^2 \sigma^2}{n(\theta \mu - 1)k}$, which is network-independent since all other terms decay faster, and thus~\SA~matches the rate of~\CSGD~(up to some constant factors); see also~\cite{DSGT_Pu,SGP_Olshevsky,pu2019sharp,SP_SPMag_2020}. It can also be verified that the network reaches an agreement at~$\mc O(1/k^2)$.

\section{Convergence Analysis} \label{sec_pt}
To aid the analysis of Theorems~\ref{th1} and~\ref{th2}, we first develop a dynamical system that characterizes~\SA~for both constant and decaying step-sizes. We find inter-relationships between the following three terms: 
\begin{enumerate}[(i)]
\item Network agreement error,~$\mbb E\|\mb x_k - B^\infty\mb x_k\|_{\bpi}^2$,
\item Optimality gap,~$\mbb E\|\ol{\mb x}_k - \mb z^*\|_{2}^2$,
\item Gradient tracking error,~$\mbb E\|\mb w_k - B^\infty\mb w_k\|_{\bpi}^2$,
\end{enumerate}
to write an LTI system of equations governing~$\SA$. For simplicity, we assume~$p = 1$. Denote~$\mb{t}_k, \mb{s}_k, \mb{c} \in \R^3$, and~$A_{\alpha}, H_k \in \R^{3\mt3}$ for all~$k$ as
\begin{align}
\mb{t}_k &:=
\left[ {\begin{array}{c}
\mbb{E}[\|\mb{x}_k - B^\infty \mb{x}_k \|^2 _{\bds{\pi}}] \\
\mbb{E}[\|\overline{\mb{x}}_k - \mb{z}^* \|_2 ^2] \\
\mbb{E}[\|\mb{w}_k - B^\infty \mb{w}_k \|^2 _{\bds{\pi}}]
\end{array} } \right],~~~~
\mb{s}_k:=
\left[ {\begin{array}{c}
\mbb{E}[\|\mb{x}_k \|_2^2] \\
0 \\
0
\end{array} } \right],~~
\mb{c} := \left[ {\begin{array}{c}
0  \\
\alpha^2 \frac{\sigma^2}{n}  \\
C_\sigma
\end{array} } \right],\nonumber\\\label{Aa_eq} 
H_k &:= \left[ {\begin{array}{c c c}
0 &0 &0  \\
h_{1} \sigma_B^k  &0 &0  \\
(h_{2}+\alpha^2 h_{3})\sigma_B^k  &0 &0
\end{array} } \right],~~
A_\alpha := \left[ {\begin{array}{c c c}
\frac{1+\sigma_B^2}{2}  &0 & \alpha^2 \frac{1+\sigma_B^2}{1-\sigma_B^2}  \\
\alpha^2 g_{1} + \alpha g_{2} & 1- \alpha \mu &0  \\
g_{3} + \alpha^2 g_{4} &\alpha^2 g_{5} & \frac{5 + \sigma_B^2}{6}
\end{array} } \right],
\end{align}
where the constants are defined as:
\begin{align*}
\begin{array}{lll}
g_{1} := \left(\frac{\ell^2 y_-^2}{n}\right) (1 + \beta \sigma_B) \overline{\pi}, \quad&g_{2} := \left( \frac{\ell^2 y_-^2}{n \mu} \right) (1 + \beta \sigma_B) \overline{\pi}, \quad&g_{3} := 4 k_2, \\
g_{4} := 2 \ell^2 y^2 k_2 k_3 (1 + \beta \sigma_B), \quad&g_{5} := 18 \ell^4 q y_-^4 y^2 \underline{\pi} ^{-1}, \quad&k_1 := \frac{1-\sigma_B^2}{3},\\
C_\sigma := \sigma^2 \left(c_{1} + \alpha^2 c_{2}\right), \quad&c_{1} := 4 q n \underline{\pi} ^{-1}, \quad&k_2 := 6 \ell^2 q y_-^2 h\\
c_{2} := 12 \ell^2 q y_-^4 y^2 k_3 \underline{\pi}^{-1}, \quad&h_{1} := y_-^2 \beta \left( \frac{\alpha \ell^2}{\mu} + \alpha^2 \ell^2 \right) (\beta+1), \quad&k_3 := \frac{2 k_1 - 3 k_2 \alpha^2}{k_1 - 2 k_2 \alpha^2},\\
h_{2} := 24 \ell^2 q y_-^4 \beta^2\ul{\pi}^{-1}, \quad&h_{3} := 12 \ell^4 q y_-^6 y^2 k_3 \beta \underline{\pi}^{-1} (\beta + 1), \quad&q := \frac{1+\sigma_B^2}{1-\sigma_B^2}.\\
\end{array}
\end{align*}

With~$\alpha \leq \left(\frac{1-\sigma_B^2}{9 \ell} \right) \frac{1}{y_-\sqrt{h}}$, we have that 
\begin{equation} \label{sys_conv}
\mb{t}_{k+1} \leq A_\alpha \mb{t}_k + H_k \mb{s}_k + \mb{c}.
\end{equation}
The derivation of the above inequality is available in Appendix~\ref{App1}. We now provide the proofs of Theorems~\ref{th1} and~\ref{th2}.

\subsection{Proof of Theorem~\ref{th1}}
From~\cite{DSGT_Pu} Lemma 5, for a~$3\times3$ non-negative, irreducible matrix~$A_\alpha \!=\! \{a_{ij}\}$ with~$\{a_{ii}\} \!<\! \lambda^*$, we have~$\rho(A_\alpha)\!<\!\lambda^*$ if and only if~$\mbox{det}(\lambda^*I_3 - A_\alpha) > 0$. For~$A_\alpha$ in~\eqref{Aa_eq},~$a_{11},a_{33}<1$ since~$\sigma_B\in[0,1)$, and~$a_{22}<1$ since~$\alpha<\frac{1}{\ell}$ and~$\ell\geq\mu$. Expanding the determinant as
\begin{align*} 
    \det(I_3 - A_{\alpha}) &= (1 - a_{11})(1 - a_{22})(1 - a_{33}) - a_{13}[a_{21} a_{32} + (1 - a_{22}) a_{31}]\\
    &= (1 - a_{22})\Big[(1 - a_{11})(1 - a_{33}) - a_{13}a_{31}\Big] - a_{13}a_{21}a_{32},
    \end{align*}
we note that if the following is true for some~$\Gamma>1$,
\begin{align}\label{a1eq}
-a_{13}a_{31} &\geq -\frac{1}{\Gamma} (1 - a_{11}) (1 - a_{33}),\\\label{a2eq}
-a_{13}a_{21}a_{32} &\geq -\frac{\Gamma-1}{\Gamma(\Gamma + 1)} (1 - a_{11}) (1 - a_{22}) (1 - a_{33}),
\end{align}
then we obtain
\begin{align*} 
    \det(I_3 - A_{\alpha}) &\geq (1 - a_{22})\Big[(1 - a_{11})(1 - a_{33}) - \frac{1}{\Gamma} (1 - a_{11}) (1 - a_{33})\Big]- \frac{\Gamma-1}{\Gamma(\Gamma + 1)} (1 - a_{11}) (1 - a_{22}) (1 - a_{33}) \\
    &\geq (1 - a_{22})(1 - a_{11})(1 - a_{33})\frac{\Gamma - 1}{\Gamma} - \frac{\Gamma-1}{\Gamma(\Gamma + 1)} (1 - a_{11}) (1 - a_{22}) (1 - a_{33}) \\
    &\geq \left(\frac{\Gamma - 1}{\Gamma + 1}\right)(1 - a_{22})(1 - a_{11})(1 - a_{33}) > 0,
\end{align*}
ensuring~$\rho(A_\alpha)<1$. We thus find the range of~$\alpha$ that satisfies~\eqref{a1eq} and~\eqref{a2eq}. Using~$\{a_{ij}\}$'s from~\eqref{Aa_eq} in~\eqref{a1eq}, we get
\begin{align*}
\alpha^2 q \left( g_{3} +  \alpha^2 g_{4} \right) &\leq \frac{1}{\Gamma} \left(\frac{1-\sigma_B^2}{2}\right) \left( \frac{1 - \sigma_B^2}{6} \right)\\
\alpha^2 k_2 (4 + \alpha^22 \ell^2 y^2  \frac{2 k_1 - 3 k_2 \alpha^2}{k_1 - 2 k_2 \alpha^2} (1 + \beta \sigma_B)) &\leq \frac{1}{12 \Gamma} \left(\frac{(1-\sigma_B^2)^3}{1+\sigma_B^2}\right)\\
\alpha^2 k_2 \left(\frac{4k_1 - 8 k_2 \alpha^2) + 2 \alpha^2 \ell^2 y^2  (1 + \beta \sigma_B)(2 k_1 - 3 k_2 \alpha^2)}{k_1 - 2 k_2 \alpha^2}\right) &\leq \frac{1}{12 \Gamma} \left(\frac{(1-\sigma_B^2)^3}{1+\sigma_B^2}\right)\\
\alpha^2 k_2 \left(4k_1 + 4k_1 \alpha^2 \ell^2 y^2  (1 + \beta \sigma_B)\right) +\frac{2 k_2 \alpha^2}{12 \Gamma} \left(\frac{(1-\sigma_B^2)^3}{1+\sigma_B^2}\right)&\leq \frac{1}{36 \Gamma} \left(\frac{(1-\sigma_B^2)^4}{1+\sigma_B^2}\right)+ 8 k_2^2 \alpha^4 + 6k_2^2 \alpha^6 \ell^2 y^2  (1 + \beta \sigma_B)
\\
\alpha^2 k_2 \left(4 k_1 + 4 k_1 \ell^2 y^2 (1+\beta \sigma_B) \alpha^2  + \frac{2}{12 \Gamma} \left(\frac{(1-\sigma_B^2)^3}{1+\sigma_B^2}\right) \right) &\leq \frac{1}{36 \Gamma} \left(\frac{(1-\sigma_B^2)^4}{1+\sigma_B^2}\right) + 8 k_2^2 \alpha^4\\
&+ 6 k_2^2 \ell^2 y^2 (1+\beta \sigma_B) \alpha^6\\
\alpha^2 k_1 k_2 \left(4 + 4 \ell^2 y^2 (1+\beta \sigma_B) \alpha^2  + \frac{1}{2 \Gamma} \left(\frac{(1-\sigma_B^2)^2}{1+\sigma_B^2}\right) \right) &\leq \frac{1}{36 \Gamma} \left(\frac{(1-\sigma_B^2)^4}{1+\sigma_B^2}\right) + 8 k_2^2 \alpha^4\\
&+ 6 k_2^2 \ell^2 y^2 (1+\beta \sigma_B) \alpha^6.
\end{align*}
We now simplify the above condition by letting~$\alpha \leq \left(\frac{1-\sigma_B^2}{9 \ell y_-} \right) \sqrt{\frac{\underline{\pi}}{\overline{\pi}}}$ in the LHS and decreasing the RHS, which leads to
\begin{align*}
\alpha^2  &\leq \frac{\frac{1}{36 \Gamma} \frac{(1-\sigma_B^2)^4}{1+\sigma_B^2} }{(2 \ell^2 y_- ^2 \underline{\pi}^{-1} \overline{\pi} (1 + \sigma_B^2)) \left(4  + 4 \left(\frac{(1-\sigma_B^2)\sqrt{\underline{\pi}}}{9 y_- \sqrt{\overline{\pi}}} \right)^2 y^2 (1+\beta \sigma_B) + \frac{1}{2 \Gamma} \left(\frac{(1-\sigma_B^2)^2}{1+\sigma_B^2}\right) \right)}\\
&= \frac{ \frac{(1-\sigma_B^2)^4}{1+\sigma_B^2} }{(\ell^2 y_- ^2 \underline{\pi}^{-1} \overline{\pi} (1 + \sigma_B^2)) \left(288\Gamma  + 288\Gamma \left(\frac{(1-\sigma_B^2)\sqrt{\underline{\pi}}}{9 y_- \sqrt{\overline{\pi}}} \right)^2 y^2 (1+\beta \sigma_B) + 36\frac{(1-\sigma_B^2)^2}{1+\sigma_B^2} \right)}\\
\impliedby \alpha^2&\leq \frac{y_-^2\frac{(1-\sigma_B^2)^4}{1+\sigma_B^2} }{\ell^2 y_- ^2 h (1 + \sigma_B^2) \left(288y_-^2 \Gamma  + 4 \Gamma (1-\sigma_B^2)^2 h^{-1} y^2 (1+\beta \sigma_B) + 36y_-^2 \frac{(1-\sigma_B^2)^2}{1+\sigma_B^2} \right)}\\
&= \frac{y_-^2(1-\sigma_B^2)^4 }{\ell^2 y_- ^2 h (1 + \sigma_B^2) \left(288 y_-^2 \Gamma (1+\sigma_B^2) + 4 \Gamma (1-\sigma_B^2)^2 h^{-1} y^2 (1+\beta \sigma_B)(1+\sigma_B^2) + 36 y_-^2 (1-\sigma_B^2)^2 \right)}.
\end{align*}
We use~$\sigma_B<1, (1 - \sigma_B^2)<1,(1+\sigma_B^2)<2, y^2 (1+\beta) \geq 1, h y_-^2 \geq 1$~and~$\Gamma h y^2(1+\beta) > 1$ leading to
\begin{align*}
\alpha^2\leq
\frac{y_-^2(1-\sigma_B^2)^4 }{2\ell^2 y_-^2 \left(612\Gamma hy_-^2 y^2 (1+\beta) + 8\Gamma h y_-^2 y^2 (1+\beta)\right)}=\frac{(1-\sigma_B^2)^4 }{1240 \ell^2 \left(\Gamma h y_-^2 y^2 (1+\beta) \right)}.
\end{align*} 
Taking square root of both sides results into
\begin{align*}
\alpha &\leq\frac{(1-\sigma_B^2)^2 }{36 \ell y_-y \sqrt{\Gamma h (1+\beta)}}.
\end{align*}

We next note that~\eqref{a2eq} holds when 
\begin{align*}
(\alpha^2 q) (\alpha^2 g_{1} + \alpha g_{2}) (\alpha^2 g_5) &\leq \frac{\Gamma-1}{\Gamma(\Gamma + 1)} \left(1 - \left( \frac{1+\sigma_B^2}{2}\right) \right) (1-(1-\alpha \mu)) \left(1- \frac{5 + \sigma_B^2}{6} \right) \\
\alpha^5 q g_5(\alpha g_{1} + g_{2}) &\leq \frac{\Gamma-1}{\Gamma(\Gamma + 1)} \left( \frac{1-\sigma_B^2}{2}\right)  (\alpha \mu)  \left(\frac{1 - \sigma_B^2}{6} \right)\\
\alpha^4 q g_5 g_{2} (1 + \alpha \mu) &\leq \frac{\Gamma-1}{\Gamma(\Gamma + 1)} \left( \frac{1-\sigma_B^2}{2}\right)^2  \left(\frac{\mu}{3}\right),
\end{align*}
which can be simplified by using~$\alpha\leq\frac
{1}{\mu}$, i.e.,
\begin{align*}
& &&\alpha^4\leq \frac{\Gamma-1}{\Gamma(\Gamma + 1)} \left(\frac{(1-\sigma_B^2)^3}{1+\sigma_B^2}\right)  \left(\frac{\mu}{24}\right) \left(\frac{\mu}{\ell^6(18 y_-^6 y^2 \underline{\pi}^{-1} \overline{\pi})(1 + \beta \sigma_B)}\right)\\
\impliedby\hspace{-2cm} & &&\alpha^4\leq \frac{\Gamma-1}{\Gamma(\Gamma + 1)} \left(\frac{(1-\sigma_B^2)^3 \mu^2}{864 \ell^6(y_-^6 y^2 \underline{\pi}^{-1} \overline{\pi})(1 + \beta \sigma_B)}\right) \\
\impliedby\hspace{-2cm} & &&\alpha\leq \frac{1}{6 \ell \sqrt{\kappa}} \Bigg [ \frac{\Gamma-1}{\Gamma(\Gamma + 1)} \left(\frac{(1-\sigma_B^2)^3}{y_-^6 y^2  h (1 + \beta \sigma_B)}\right)  \Bigg ]^\frac{1}{4},
\end{align*}
for which it is sufficient to have
\begin{equation}
\alpha \leq \frac{(1-\sigma_B^2)^{3/4}}{12 \ell \sqrt{\kappa}} \left(\frac{\Gamma - 1}{\Gamma^2 y_-^6 y^2  h (1 + \beta)}\right)^\frac{1}{4}.
\end{equation}
We next select the minimum of all the bounds on step-size,
\begin{align*}
    \alpha &\leq \min \left\{\frac{1-\sigma_B^2}{9 \ell y_-\sqrt{h}}, \frac{(1-\sigma_B^2)^2}{36 \ell y_- y \sqrt{\Gamma h (1+\beta)}}, \frac{(1-\sigma_B^2)^{3/4}}{12 \ell \sqrt{\kappa}} \left(\frac{\Gamma - 1}{\Gamma^2 y_-^6 y^2  h (1 + \beta)}\right)^\frac{1}{4}\right\} \\
    \impliedby \alpha &\leq \frac{(1-\sigma_B^2)^2}{36 \ell \sqrt{\kappa}} \cdot \min \left\{\left ( \frac{1}{ \tau \Gamma } \right )^\frac{1}{2}, \left(\frac{\Gamma - 1}{\tau \Gamma^2}\right)^\frac{1}{4}\right\}\\ 
    \impliedby \alpha &\leq \frac{(1-\sigma_B^2)^2}{36 \ell \sqrt{\kappa}} \cdot \frac{1}{\sqrt{\tau \Gamma}} \cdot \min \left\{1, \left(\Gamma - 1\right)^\frac{1}{4}\right\},
\end{align*}
where~$\tau := y_-^6 y^2 h(1+\beta)$. We note that the above is true for all~$\Gamma>1$ and~$\min \left\{1, \left(\Gamma - 1\right)^\frac{1}{4}\right\}$~is maximized at~$\Gamma = 2$. Hence, for a largest possible~$\alpha$, that is feasible given our bound, we select~$\Gamma = 2$, which leads to
\begin{align*}
    \alpha &\leq \frac{1}{\ell\sqrt\kappa}\cdot\frac{(1-\sigma_B^2)^2}{51\sqrt{\tau}}.
\end{align*}
Thus, when~$\alpha$ follows the above relation, we have~$\rho(A_\alpha)<1$ and using the linear system recursion in~\eqref{sys_conv}, we get
\begin{equation} \label{tk_sup}
\lim_{k\ra\infty}\mb{t}_{k+1} \leq (I_3-A_\alpha )^{-1}\mb{c},
\end{equation}
since~$\lim_{k\ra\infty} H_k$ is a zero matrix. The first two elements in the R.H.S (vector) of~\eqref{tk_sup} can be manipulated as follows:
\begin{align}
[(I_3 - A_\alpha)^{-1} \mb{c}]_1 &= \frac{a_{13}a_{32} \frac{\alpha^2 \sigma^2}{n} + a_{13}(1 - a_{22})C_{\sigma}}{\det(I_3 - A_\alpha)} \nonumber\\
&\leq \left(\frac{\Gamma + 1}{\Gamma - 1}\right) \frac{a_{13}}{(1 - a_{11})(1 - a_{22})(1 - a_{33})} \left[a_{32} \frac{\alpha^2 \sigma^2}{n} + (1 - a_{22})C_{\sigma} \right] \nonumber
\\
&\leq \left(\frac{\alpha^2\left(\frac{1+\sigma_B^2}{1-\sigma_B^2} \right)}{\left(\frac{1-\sigma_B^2}{2}\right)(\alpha \mu )\left(\frac{1-\sigma_B^2}{6}\right)}\right) \left[  \alpha^2(18 \ell^4 y_-^4 y^2 \underline{\pi}^{-1})\left( \frac{1+\sigma_B^2}{1-\sigma_B^2} \right) \left( \frac{\alpha^2 \sigma^2}{n} \right) + (\alpha \mu ) C_\sigma \right] \nonumber\\
&\leq \left(\frac{12 \alpha \left(1+\sigma_B^2\right)}{ \mu \left(1-\sigma_B^2\right)^3}\right) \left[ 18 \alpha^4 \ell^4 y_-^4 y^2 \underline{\pi}^{-1} \left( \frac{1+\sigma_B^2}{1-\sigma_B^2} \right) \left( \frac{\sigma^2}{n} \right) + \alpha \mu (4 \sigma^2 n \underline{\pi}^{-1})\left( \frac{1+\sigma_B^2}{1-\sigma_B^2} \right) \right] \nonumber\\
&=  \alpha^5 \left( \frac{\ell^4 \sigma^2}{n \mu} \right) \left(\frac{216 y_-^4 y^2 \underline{\pi}^{-1}  \left(1+\sigma_B^2\right)^2}{ \left(1-\sigma_B^2\right)^4}\right) + \alpha^2 ( n \sigma^2)  \left(\frac{48 \underline{\pi}^{-1} \left(1+\sigma_B^2\right)^2}{ \left(1-\sigma_B^2\right)^4}\right) \nonumber\\
\label{asymp}
&= \frac{\alpha^5}{(1 - \sigma_B^2)^4} \mc{O}\left( \frac{\ell^4 \sigma^2}{n \mu } \right) + \frac{\alpha^2}{(1-\sigma_B^2)^4} \mc{O} \left(n \sigma^2 \right);
\end{align}
\begin{align}
[(I_3 - A_\alpha)^{-1} \mb{c}]_2 &= \frac{[(1 - a_{11})(1 - a_{33}) - a_{13}a_{31}]\frac{\alpha^2 \sigma^2}{n} + (a_{13}a_{21})C_{\sigma}}{\det(I_3 - A_\alpha)} \nonumber\\
&\leq \frac{\Gamma + 1}{\Gamma}\left(\frac{\alpha^2 \sigma^2}{n(1 - a_{22})} \right) + \left( \frac{\Gamma + 1}{\Gamma-1} \right) \left(\frac{a_{13}a_{21}C_{\sigma}}{(1 - a_{11})(1 - a_{22})(1 - a_{33})}\right) \nonumber\\
&\leq \frac{\alpha^2 \sigma^2 }{n(\alpha \mu)} + \frac{\alpha^2\left(\frac{1+\sigma_B^2}{1-\sigma_B^2} \right) \left(\alpha^2 \left(\frac{\ell^2 y_-^2 (1 + \beta \sigma_B) \ol{\pi}}{n} \right) + \alpha \left( \frac{\ell^2 y_-^2 (1 + \beta \sigma_B)\ol{\pi} }{n \mu} \right) \right) C_\sigma}{\left(\frac{1-\sigma_B^2}{2}\right)(\alpha \mu )\left(\frac{1-\sigma_B^2}{6}\right)} \nonumber\\
&= \frac{\alpha \sigma^2 }{n \mu} + \frac{12 \alpha \left(1+\sigma_B^2\right)^2 \left(\alpha^2 (\ell^2 y_-^2 (1 + \beta \sigma_B) \ol{\pi}) + \alpha \left( \frac{\ell^2 y_-^2 (1 + \beta \sigma_B)\ol{\pi} }{\mu} \right) \right) (4 \sigma^2 n \underline{\pi}^{-1})}{n \mu (1-\sigma_B^2)^4} \nonumber\\
\label{asymp2}
&=\alpha \mc{O} \left(\frac{\sigma^2 }{n \mu} \right) + \frac{\alpha^2}{(1-\sigma_B^2)^4} \mc{O} \left(\frac{\ell^2 \sigma^2 }{\mu^2}\right). 
\end{align}
Finally, the mean network error, defined as~$\mb e_k:= \frac{1}{n}\mbb{E}\left[\|\mb{z}_k - \mb1_n\mb{z}^* \|_2 ^2\right]$, is given by
\begin{align} \label{ztozstar}
\mb{e}_k &\leq \frac{3 y_-^2 \ol{\pi}}{n} \mbb{E}[\|\mb{x}_k - B^\infty \mb{x}_k \|_{\bds{\pi}}^2] + 3 y_-^2 \beta^2 \mbb{E}[\|\mb{z}^* \|_2^2]  \sigma_B^{2k} + 3 y_-^2 y^2 \mbb{E}[\|{\mb{x}}_k - \mb{1}_n\mb{z}^* \|^2_2]. 
\end{align}
Notice that the second term of~\eqref{ztozstar}~vanishes asymptotically. Using~\eqref{asymp} and~\eqref{asymp2}, we further have 
\begin{align} \label{ztozstar2}
\limsup_{k \rightarrow \infty} \mb{e}_k &\leq \frac{3 y_-^2 \ol{\pi} \alpha^5}{(1-\sigma_B^2)^4} \mc{O} \left( \frac{\ell^4 \sigma^2}{n^2 \mu} \right) + \frac{3 y_-^2 \ol{\pi} \alpha^2}{(1-\sigma_B^2)^4} \mc{O} \left( \sigma^2 \right)+ \frac{3 y_-^2 y^2 \alpha^2}{(1-\sigma_B^2)^4} \mc{O} \left( \frac{\ell^2\sigma^2}{\mu^2} \right) + 3y_-^2 y^2 \alpha \mc{O} \left(\frac{\sigma^2 }{n \mu} \right).
\end{align}
and the theorem follows by dropping the higher order term of~$\alpha$ and noting that~$\frac{\ell^2}{\mu^2}\geq 1$. $\hfill\square$

\begin{cor}\label{cor2}
For all~$k,\exists b \in \mbb{R}$, such that~$\mbb{E}[\|\mb{x}_k\|_2^2] \leq b.$
\end{cor}
The proof follows from Theorem~\ref{th1}.

\subsection{Proof of Theorem 2}
Let~$P_k := \mbb{E}[\|\mb{x}_k - B^\infty \mb{x}_k \|^2 _{\bds{\pi}}]$,~$ Q_k := \mbb{E}[\|\overline{\mb{x}}_k - \mb{z}^* \|_2 ^2]$ and~$ R_k := \mbb{E}[\|\mb{w}_k - B^\infty \mb{w}_k \|^2 _{\bds{\pi}}]$. To show that
\begin{align} \label{th2_hypothetis}
P_{k} \leq \frac{\wt{P}}{(m+k)^2}, \qquad Q_{k} \leq \frac{\wt{Q}}{(m+k)}, \qquad R_{k} \leq \wt{R},
\end{align}
for all~$k\geq 0,$ it suffices to show that the R.H.S of~\eqref{sys_conv}, with a decaying step-size~$\alpha_k < \left(\frac{1-\sigma_B^2}{6 \ell} \right) \frac{1}{y_-\sqrt{(1+\sigma_B^2)h}}$, follows the above bounds. We develop the proof by induction. Consider~\eqref{sys_conv} for $k=0$, i.e.,
\[
{A_{\alpha_0}\mb t_0 + H_0\mb s_0 + \mb c}
\]
with~$\alpha_0=\frac{\theta}{m}$, and therefore~$m>\frac{6 \ell \theta y_- \sqrt{(1+\sigma_B^2)h}}{1-\sigma_B^2}$, to obtain the following conditions:
\begin{subequations}\label{th2_PQR}
\begin{align} 
\wt{R} &\leq \left(\frac{1-\sigma_B^2}{\theta^2(1+\sigma_B^2)} \right) \left(\frac{m^2}{(m+1)^2} - \frac{1+\sigma_B^2}{2} \right)\wt{P}, \label{th2_R_const}\\
\wt{Q} &\geq \left[ \left(\frac{\theta}{m} + \frac{1}{\mu} \right) \left( \frac{\theta \ell^2 E_1}{m n \left(\theta \mu - 1 \right)} \right) \wt{P} + \frac{n m^2 K_1 b + \theta^2 \sigma^2}{n \left(\theta \mu - 1 \right)} \right], \label{th2_Q_const}\\
\wt{R} &\geq \frac{6}{1-\sigma_B^2} \left[\left(\frac{E_2}{m^2} \right) \wt{P} + \left(\frac{\theta^2 \ell^4 E_{3}}{m^3} \right) \wt{Q} + K_2 b + C_0 \right]. \label{th2_R_const2}
\end{align}
\end{subequations}
where~$E_1,E_2,E_3$ are defined in the following. It can be~verified, that the above conditions hold if and only if
\begin{align*}
\frac{1-\sigma_B^2}{\theta^2(1+\sigma_B^2)} \left(\frac{m^2}{(m+1)^2} - \frac{1+\sigma_B^2}{2} \right)\wt{P} &> \frac{6}{1-\sigma_B^2} \left[\frac{E_2 \wt{P}}{m^2} + \left(\frac{\theta^2 \ell^4 E_{3}}{m^3} \right) \left(\frac{\theta}{m} + \frac{1}{\mu} \right) \left( \frac{\theta \ell^2 E_1 \wt{P}}{m n \left(\theta \mu - 1 \right)} \right) \right] \\
\frac{(1-\sigma_B^2)^2}{6 \theta^2(1+\sigma_B^2)} \left(\frac{1-\sigma_B^2}{2} - \frac{2 m + 1}{(m+1)^2} \right) &> \frac{E_2}{m^2} + \left(\frac{\theta^3 \ell^6 E_1 E_3}{m^4 n \left(\theta \mu - 1 \right) } \right)
\left(\frac{\theta}{m} + \frac{1}{\mu} \right)
\end{align*}
and~$\wt{Q} = \max \left \{m Q_0, D_6 \right \}$, where~$\wt{P}$ and~$\wt{R}$ follow the constraints in~\eqref{th2_R_const},~\eqref{th2_R_const2}, and~$\wt{R}>R_0$. We use~${\mbb{E}\|\mb x_k\|_2^2 < b}$, for some~${b>0}$, which follows from Theorem~\ref{th1}. Thus,~$\wt{P}$ is selected as~$\wt{P} = \max \left \{ m^2 P_0, \frac{R_0}{D_1}, \frac{D_3}{D_1 - D_2}, \frac{D_5}{D_1 - D_4}  \right \}$, where
{ 
\begin{align*}
\begin{array}{ll}
C_0 := 4 \sigma^2 q \underline{\pi} ^{-1} \left(n  + 3 \left(\frac{\theta^2 \ell^2 y_-^4 y^2}{m^2} \right) \left(\frac{2 m^2 k_1 - 3 k_2 \theta^2 }{ m^2 k_1 - 2 k_2 \theta^2} \right)\right), \quad&D_2 := \frac{6 E_2}{m^2(1-\sigma_B^2)}, \nonumber\\
D_1 := \left(\frac{1-\sigma_B^2}{\theta^2(1+\sigma_B^2)} \right) \left(\frac{1-\sigma_B^2}{2} - \frac{2 m + 1}{(m+1)^2} \right), \quad&D_4 := \left[\frac{6 E_1}{1-\sigma_B^2} \right] \left(\frac{\theta^3 \ell^6 E_3}{m^4 n \left(\theta \mu - 1 \right)} \right) \left(\frac{\theta}{m} + \frac{1}{\mu} \right) + D_2, \nonumber\\
D_3 := \left[\frac{6}{1-\sigma_B^2} \right] \left[ \left(\frac{\theta^2 \ell^4 E_3}{m^3} \right) \|\overline{\mb{x}}_0 - \mb{z}^* \|_2 ^2 + C_0 + K_2 b \right], \quad&E_1 := (1 + \beta \sigma_B) y_-^2 \overline{\pi}, \nonumber\\
D_5 := \left[\frac{6}{1-\sigma_B^2} \right] \left[ \left(\frac{\theta^2 \ell^4 E_3}{m^3 n \left(\theta \mu - 1 \right)} \right) (\theta^2 \sigma^2 + n m^2 K_1 b) + C_0 + K_2 b \right], \quad&E_3 := 18 q y_-^4 y^2 \ul{\pi}^{-1},\\
D_6 := \left[\frac{1}{n(\theta \mu -1)}\right] \left[ \left( \frac{\theta}{m} + \frac{1}{\mu} \right) \left( \frac{\theta \ell^2 E_1}{m} \right) \wt{P} + \theta^2 \sigma^2 + n m^2 K_1 b \right], \quad&K_1 := K_3 \left( \frac{\theta \ell^2 }{\mu m} + \frac{\theta^2 \ell^2}{m^2} \right),\\
K_2 := \frac{12 \ell^2 q y_-^4 \beta}{\ol{\pi}} \left(2 \beta + \frac{\theta^2 \ell^2 y_-^2 y^2 (\beta+1)}{m^2} \left(\frac{2 m^2 k_1 - 3 k_2 \theta^2}{m^2 k_1 - 2 k_2 \theta^2} \right) \right), \quad&K_3 := y_-^2 \beta (\beta+1), \\
E_2 := 4 k_2 + \left( \frac{2 \ell^2 y^2 k_2 \theta^2}{m^2} \right) \left(\frac{2k_1 m^2 - 3k_2 \theta^2}{k_1 m^2 - 2k_2 \theta^2} \right)(1+ \beta \sigma_B).
\end{array} 
\end{align*}
}Thus, we conclude that~\eqref{th2_hypothetis} holds for~$k=0$ when the corresponding conditions on~$\wt P,\wt Q, \wt R,$ and~$m$ are met. Next, assume that~\eqref{th2_hypothetis} holds for some~$k$, it can be verified that it automatically holds for~$k+1$ with the same conditions on~$\wt P,\wt Q, \wt R,$ and~$m$ that are derived for~$k=0$.

We next improve~$Q_k$ to establish the network-independence. Pick~$\wt{S}$~large enough such that~$ \forall k \geq \wt{S}, {\sigma_B^k \leq \frac{1}{n(m+k)^2}}$. Then using the decaying step-size in~\eqref{sys_conv}, we have
\begin{align*}
Q_{k+1} \leq \left(1 - \frac{\theta \mu}{m+k} \right) Q_{k} + \frac{2 \theta \ell^2 (E_1 \widetilde{P} + K_3 b)}{n \mu (m+k)^3} +  \frac{\theta^2 \sigma^2}{n(m+k)^2},
\end{align*}
which leads to
\begin{align} \label{qk}
Q_k &\leq  \prod_{t=0}^{k-1} \left(1 - \frac{\theta \mu}{m+t} \right) Q_0 +  \sum_{t=0}^{k-1} \prod_{j=t+1}^{k-1} \frac{m+j-\theta}{m+j} \left[\frac{2 \theta \ell^2 (E_1 \widetilde{P} + K_3 b) }{n \mu (m+t)^3} +  \frac{\theta^2 \sigma^2}{n(m+t)^2}\right],
\end{align}
From~\cite{pu2019sharp}, we have
\begin{align*}
\prod_{t=0}^{k-1} \left(1 - \frac{\theta \mu}{m+t} \right) \leq \frac{m^{\theta \mu}}{(m+k)^{\theta \mu}}, \qquad \prod_{j=t+1}^{k-1} \left(1 - \frac{\theta \mu}{m+j} \right) \leq \frac{(m+t+1)^{\theta \mu}}{(m+k)^{\theta \mu}};
\end{align*}
Using the above relations and in~\eqref{qk},
\begin{align*}
Q_k &\leq \frac{m^{\theta \mu}}{(m+k)^{\theta \mu}} Q_0 + \frac{4\theta \ell^2 (E_1 \widetilde{P} + K_3 b)}{n \mu(m+k)^{\theta \mu}}\sum_{t=0}^{k-1} (m+t)^{\theta \mu -3 } + \frac{2 \theta^2 \sigma^2}{n(m+k)^{\theta \mu}}\sum_{t=0}^{k-1} (m+t)^{\theta \mu-2}\\
&\leq \frac{m^{\theta \mu}}{(m+k)^{\theta \mu}} Q_0 + \frac{4\theta \ell^2 (E_1 \widetilde{P} + K_3 b)}{n \mu(m+k)^{\theta \mu}}\int_{t=-1}^{k} (m+t)^{\theta \mu -3 } dt + \frac{2 \theta^2 \sigma^2}{n(m+k)^{\theta \mu}}\int_{t=-1}^{k} (m+t)^{\theta \mu-2} dt\\
&\leq \frac{2 \theta^2 \sigma^2}{n(\theta \mu - 1)(m+k)} + \frac{m^{\theta \mu}}{(m+k)^{\theta \mu}} Q_0 + \max \left \{ \frac{4\theta \ell^2 (E_1 \widetilde{P} + K_3 b)}{n \mu (\theta \mu - 2)(m+k)^2}, \frac{4\theta \ell^2 (E_1 \widetilde{P} + K_3 b)(m-1)^{\theta \mu -2}}{n \mu (2 - \theta \mu)\mu (m+k)^{\theta \mu}} \right\},
\end{align*}
and the theorem follows by~\eqref{e2} in Lemma~\ref{lem_e2} and by noting that the~$\frac{1}{(m+k)}$ term in~$Q_k$ is network independent.~$\hfill\square$
\newpage
\section{Numerical Simulations} \label{sec_ne}
In this section, we illustrate~$\SA$ and compare its performance with related algorithms over directed graphs, i.e.,~\GP~\cite{opdirect_Tsianous,GP_neidch},~\ADDOPT~\cite{add-opt,DIGing}, and~\SGP~\cite{SGP_nedich,SGP_ICML,SGP_Olshevsky}. Recall that~\GP~and~\ADDOPT~are batch algorithms and operate on the entire local batch of data at each node. In other words, the true gradient~$\nabla f_i$ is used at each node to compute the algorithm updates. In contrast,~\SGP~and~\SA~employ a stochastic gradient~$\nabla \wh{f}_i(\cdot)=\nabla f_{i,s_k^i}(\cdot)$, where~$s_k^i$ is chosen uniformly at random from the index set~$\{1,\ldots,m_i\}$ at each node~$i$ and each time~$k$. It can be verified that this choice of stochastic gradient satisfies the~\SFO~setup in Assumption~\ref{SFO_assump}. The numerical experiments are described next. 

\subsection{Logistic Regression: Strongly convex}
We now show the numerical experiments for a binary classification problem to classify hand-written digits~$\{3,8\}$ from the MNIST dataset. In this setup, there are a total of~$N=12,\!000$ labeled images for training and each node~$i$ possesses a local batch with~$m_i$ training samples. The~$j$-th sample at node~$i$ is a tuple~$\{\mb{x}_{i,j},y_{i,j}\}\subseteq\mbb{R}^{784}\times\{+1,-1\}$ and the local logistic regression cost function~$f_i$ at node~$i$ is
given by 
\begin{align*}
f_i = \frac{1}{m_i}\sum_{j=1}^{m_i}\ln\left[1+\exp\left\{-(\mb{b}^\top\mb{x}_{i,j}+c)y_{i,j}\right\}\right]+\frac{\lambda}{2}\|\mb{b}\|_2^2, \end{align*}
which is smooth and strongly convex because of the addition of the regularizer~$\lambda$. The nodes cooperate to solve the following decentralized optimization problem:
\begin{equation*}
\operatorname*{min}_{\mb{b}\in\mathbb{R}^{784},\:c\in\mathbb{R}}F(\mb{b},c) = \frac{1}{n}\sum_i f_i.
\end{equation*}
For all algorithms, the step-sizes are  hand-tuned for best performance. The column stochastic weights are chosen such that~$b_{ji}=|\mc{N}_i^{{\scriptsize\mbox{out}}}|^{-1}$, for each~$j\in\mc{N}_i^{{\scriptsize\mbox{out}}}$. 

\begin{figure}[!h]
\centering
\subfigure{\includegraphics[width=2.8in]{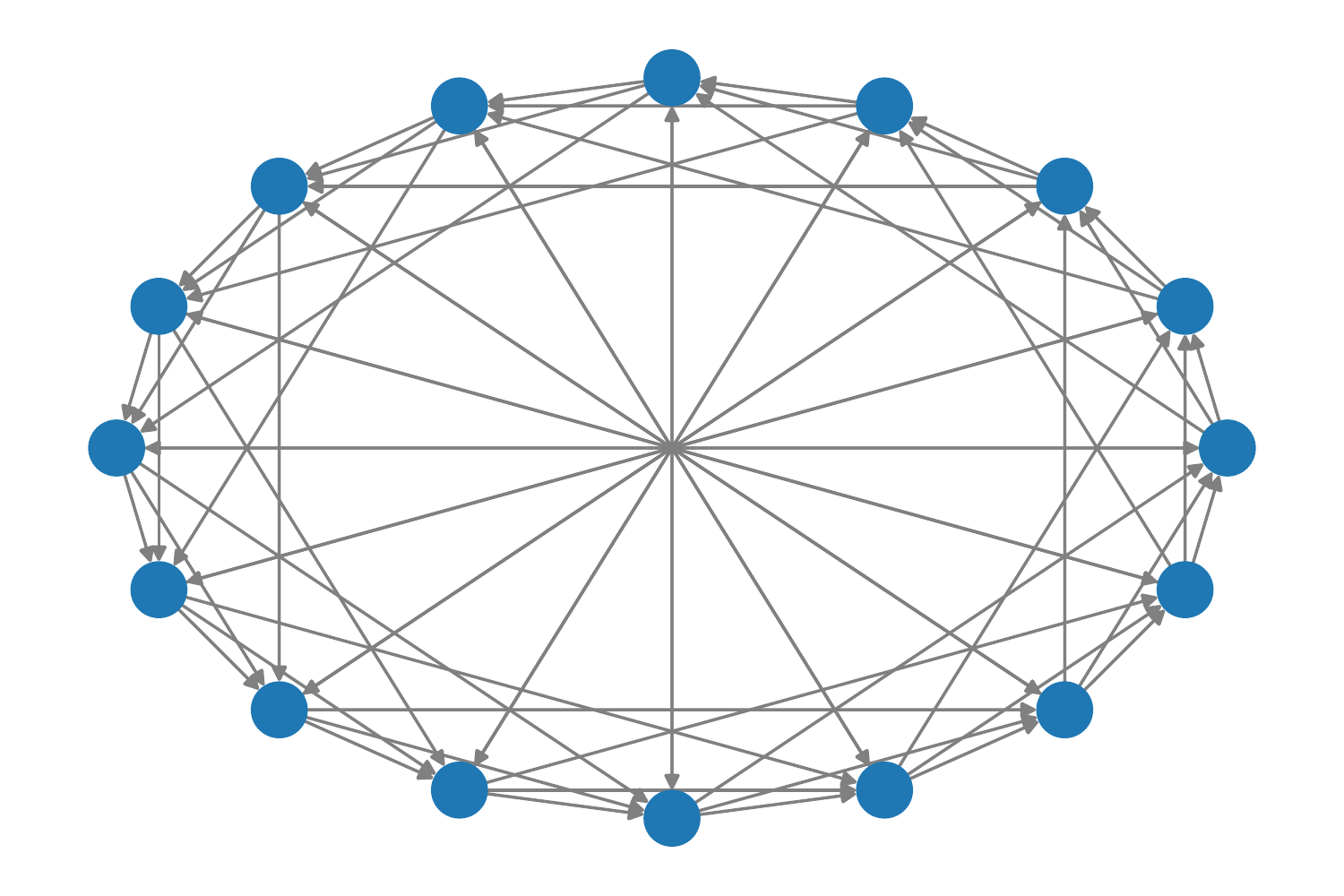}}
\hspace{0cm}\subfigure{\includegraphics[width=2.8in]{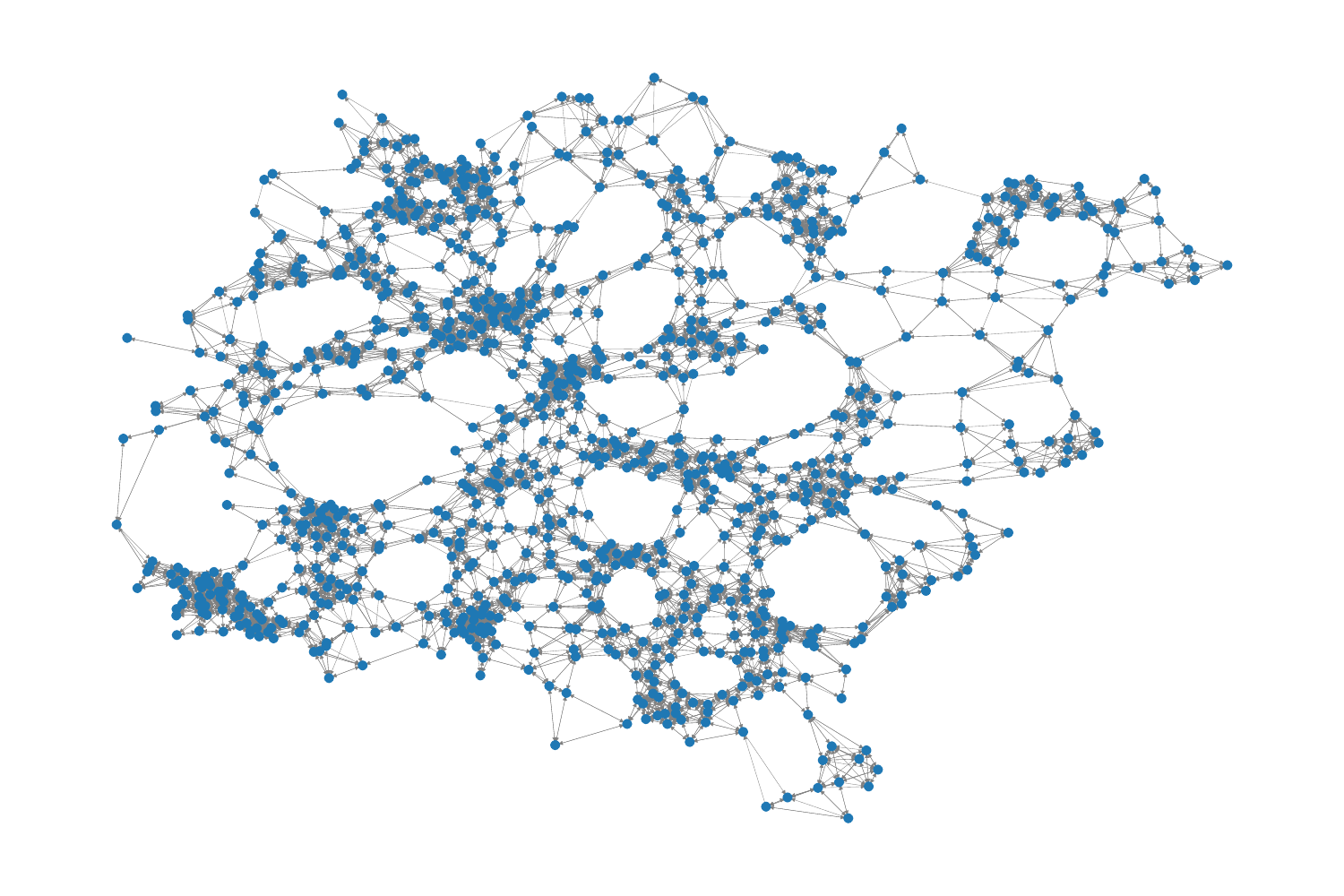}}
\caption{(Left) Directed exponential graph with~$n=16$ nodes. (Right) geometric graph with~$n=1000$ nodes}
\label{G1}
\end{figure}
{\bf Structured training setup--Data-centers:}
We choose an exponential graph with~$n=16$ nodes (Fig.~\ref{G1}, left) to model a highly structured communication graph mimicking, for example, a data center where the data is typically evenly divided among the nodes. In particular, we choose~$m_i=N/n=750$ training images at each node~$i$. Performance  comparison is provided in Fig.~\ref{BalancedGraph2}, for a constant step-size, and in  Fig.~\ref{DecayingStepSize} (left), for decaying step-sizes, where we plot the optimality gap~$F(\overline{\mb x}_k)-F(\mb z^*)$ versus the number of epochs. Each epoch represents~$N/n = 750$ stochastic gradient evaluations implemented (in parallel) at each node. Recall that~\SA~adds gradient tracking to~\SGP~and in this balanced data scenario, its performance is virtually indistinguishable from~\SGP, while their batch  counterparts are much slower. \ADDOPT~however converges linearly to the exact solution as can be observed in Fig.~\ref{BalancedGraph2}~(right) over a longer number of epochs.
\begin{figure}[!h]
\centering
\subfigure{\includegraphics[width=2.8in]{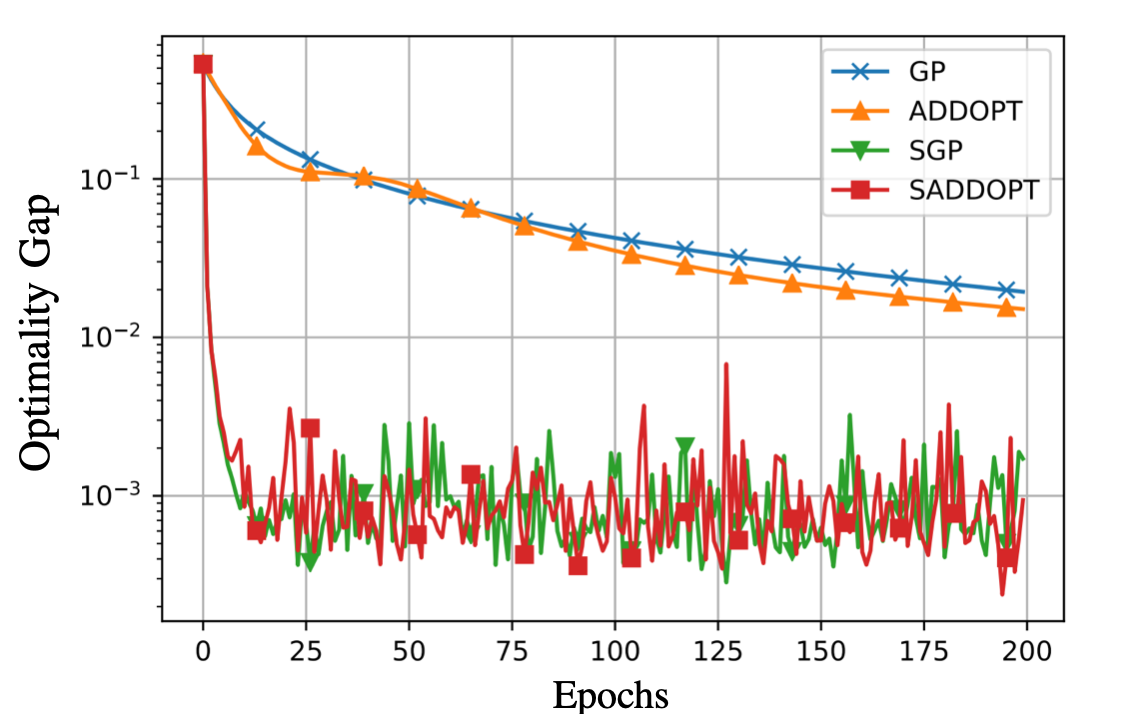}}
\subfigure{\includegraphics[width=2.8in]{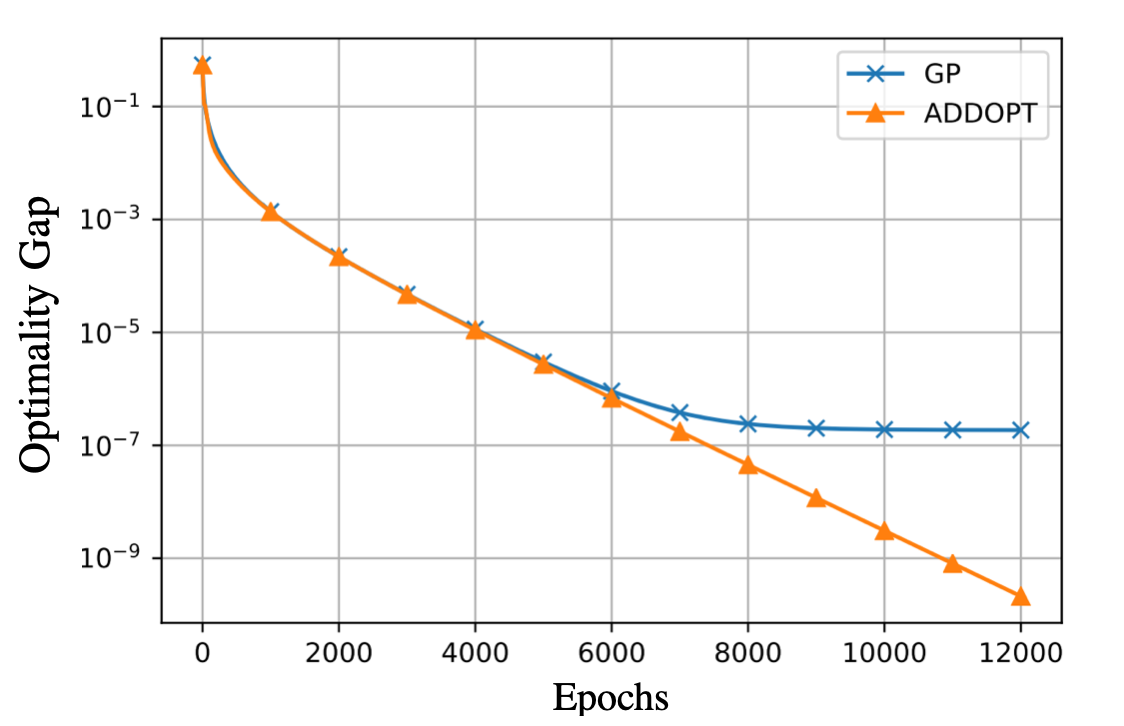}}
\caption{(Left) Balanced data and constant step-sizes for all algorithms: Performance comparison over the exponential graph with~$n=16$ nodes and~$m=750$ data samples per node. (Right) Linear convergence of~\ADDOPT~shown over a longer number of epochs.}
\label{BalancedGraph2}
\end{figure}

{\bf Ad hoc training setup--Multi-agent networks:}
We next consider a large-scale nearest-neighbor (geometric) digraph with~$n=1,\!000$ nodes (Fig.~\ref{G1}, right) that models, for example, ad hoc wireless multi-agent networks, where the agents typically possess different sizes of local batches depending on their locations and local resources; see Fig.~\ref{UnbalancedData} (left) for an arbitrary data distribution across the agents. Performance comparison is shown in Fig.~\ref{UnbalancedData} (right), for a constant step-size, and in Fig.~\ref{DecayingStepSize} (right), for decaying step-sizes. Each epoch represents $N/n = 12$ component gradient evaluations (in parallel) at each node. When the data is unbalanced, the addition of gradient tracking in~\SA~results in a significantly improved performance than~\SGP.
\begin{figure}[!h]
\centering
\subfigure{\includegraphics[width=2.8in]{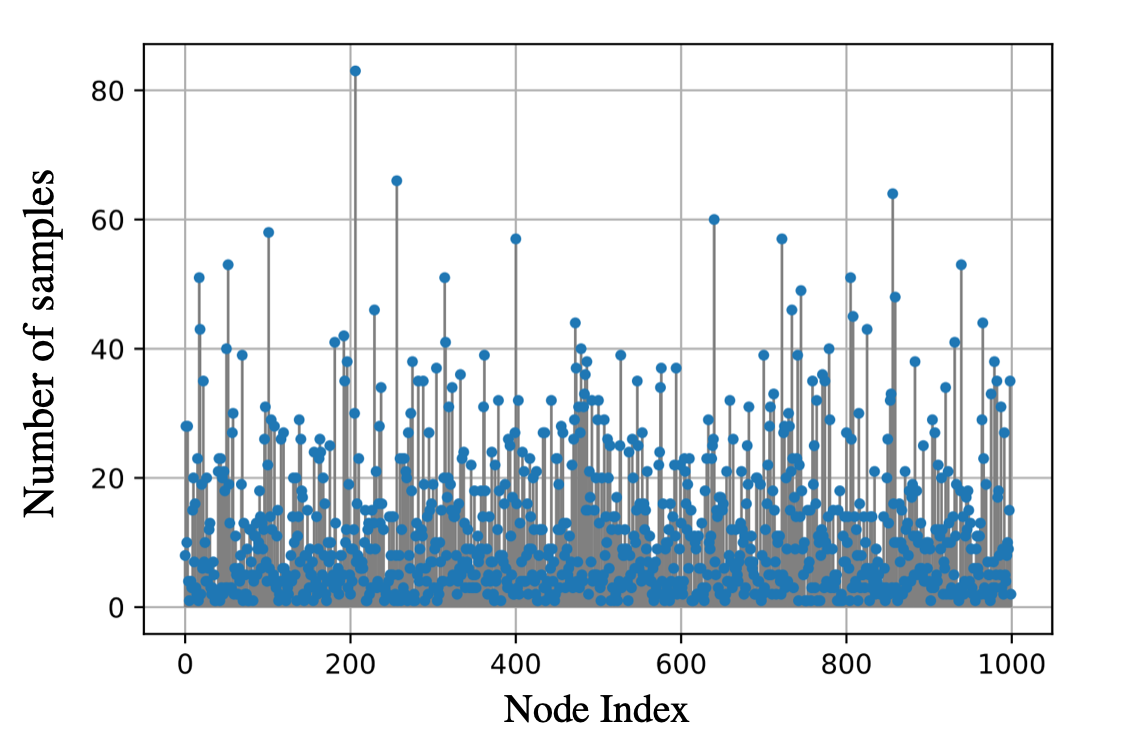}}
\subfigure{\includegraphics[width=2.8in]{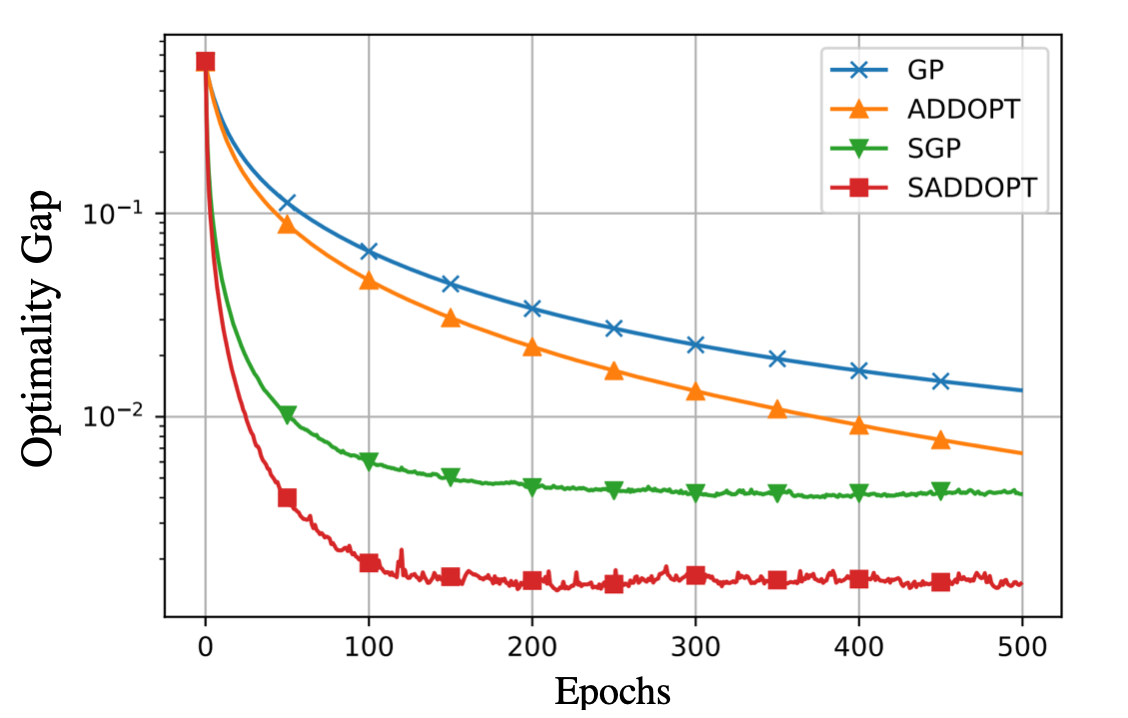}}
\caption{Performance comparison (right), over the directed geometric graph in Fig.~\ref{G1} (right), with an unbalanced data distribution (left) and constant step-sizes for all algorithms.}
\label{UnbalancedData}
\end{figure}

Comparing the structured and ad hoc training scenarios, we note that gradient tracking does not show a noticeable improvement over the balanced data scenario but results in a superior performance when the data distribution is unbalanced. This is because the convergence~\eqref{ztozstar2} of~$\SA$ (similar to its undirected counterpart~\cite{DSGT_Pu}) does not depend on the heterogeneity of local data batches as opposed to~\SGP. A detailed discussion along these lines can be found in~\cite{GT_SAGA_SPM}.
\begin{figure}[!h]
\centering
\subfigure{\includegraphics[width=2.8in]{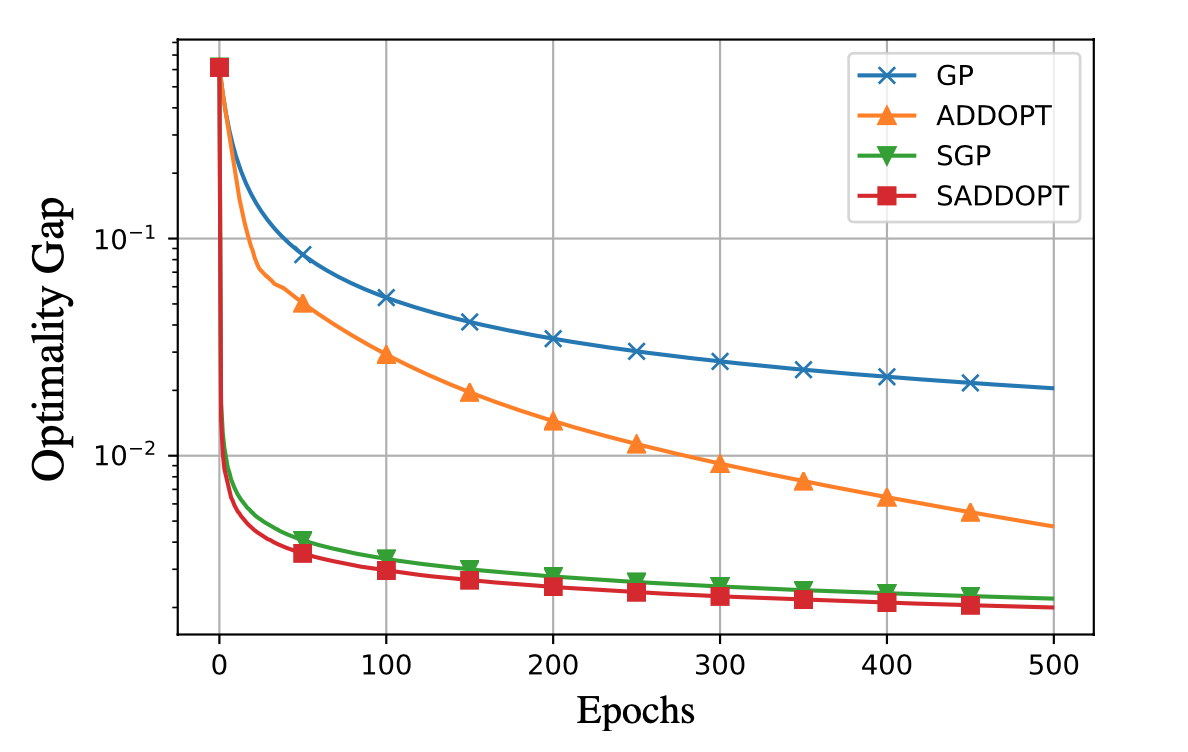}}
\subfigure{\includegraphics[width=2.8in]{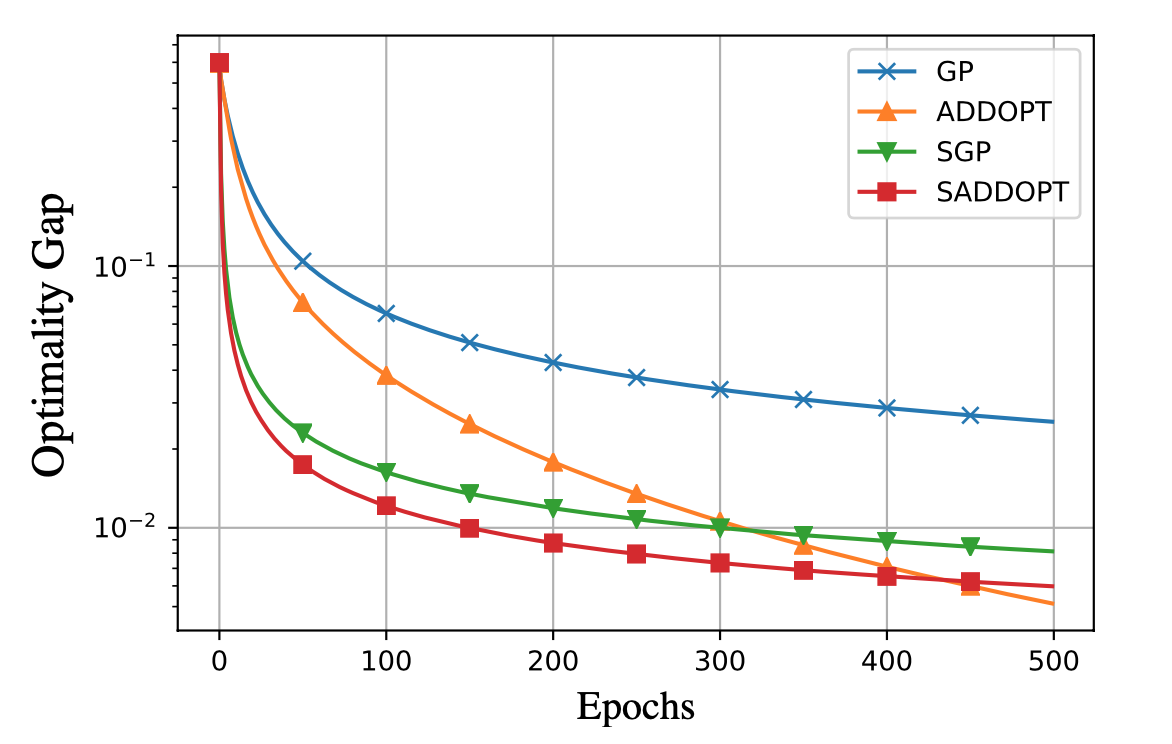}}
\caption{Performance comparison for exact convergence (decaying step-sizes for~\SA~and~\SGP,~and constant step-size for~\ADDOPT): (Left) Directed exponential graph with balanced data.  (Right) Directed geometric graph with unbalanced data.}
\label{DecayingStepSize}
\end{figure}

\subsection{Neural networks: Non-convex}
Finally, we compare the performance of the stochastic algorithms discussed in this report for training a distributed neural network optimizing a non-convex problem with constant step-sizes of the algorithms. Each node has a local neural network comprising of one fully connected hidden layer of 64 neurons learning 51,675 parameters. We train the neural network to for a multi-class classification problem to classify ten classes in MNIST~$\{0, \cdots, 9\}$ and CIFAR-10~$\{\mbox{``airplanes"}, \cdots , \mbox{``trucks"}\}$~datasets. Both have~60,000 images in total and~6,000 images per class. The data samples are divided randomly and equally over a 500 node directed geometric graph shown in Fig.~\ref{G2}.
\begin{figure}[!h]
\centering
\includegraphics[width=4.0in]{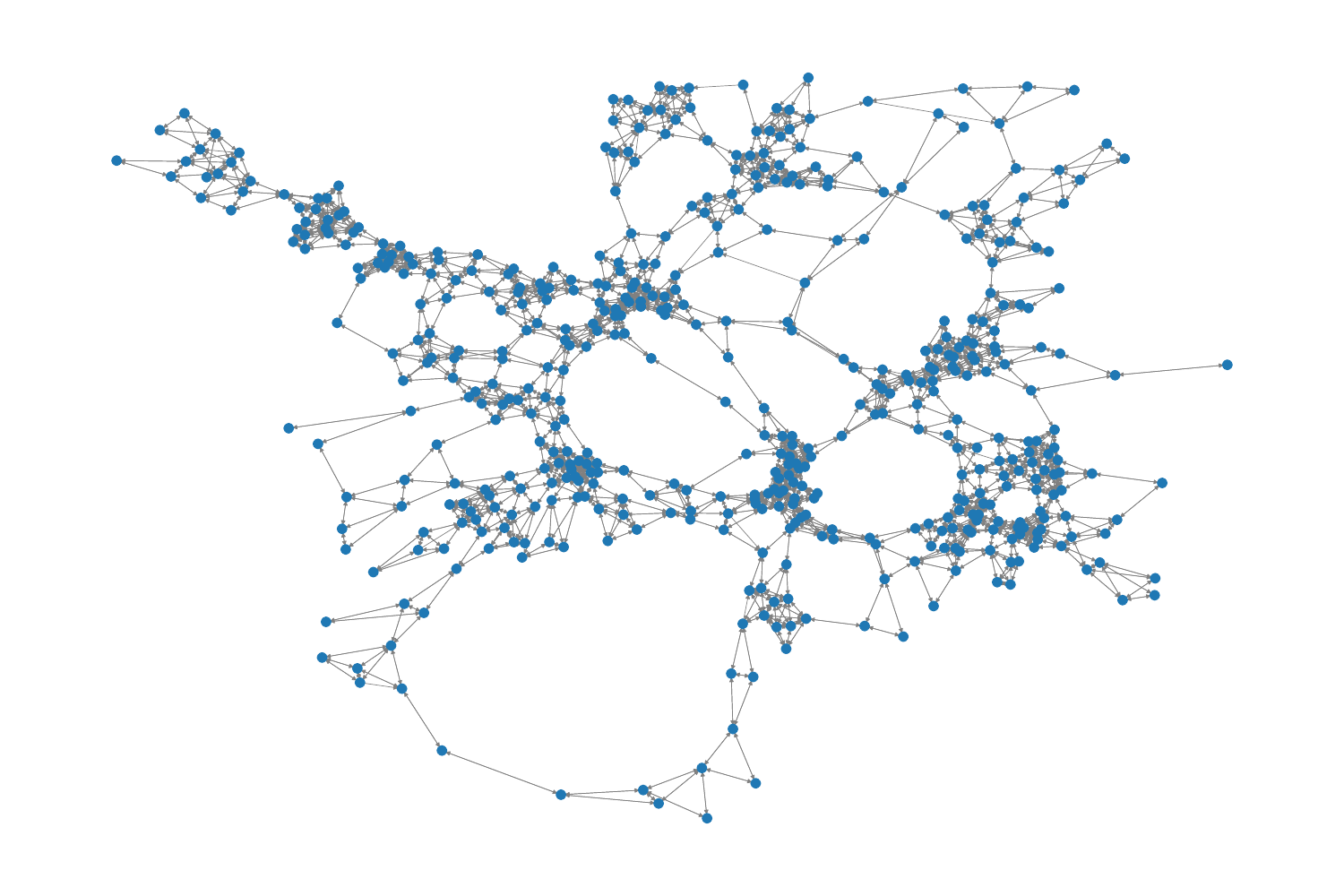}
\caption{Directed geometric graph with~$n = 500$~nodes.}
\label{G2}
\end{figure}
We show the loss~$F(\mb{\ol{x}}_k)$~and test accuracy of~\SGP~and~\SA~with respect to epochs over the MNIST dataset in Fig.~\ref{NN}. Similarly, Fig.~\ref{NN2}~illustrates the performance for the CIFAR-10 dataset. We observe that adding gradient tracking in~\SGP~improves the transient and steady state performance in these non-convex problems.
\begin{figure}[!h]
\centering
\subfigure{\includegraphics[width=2.8in]{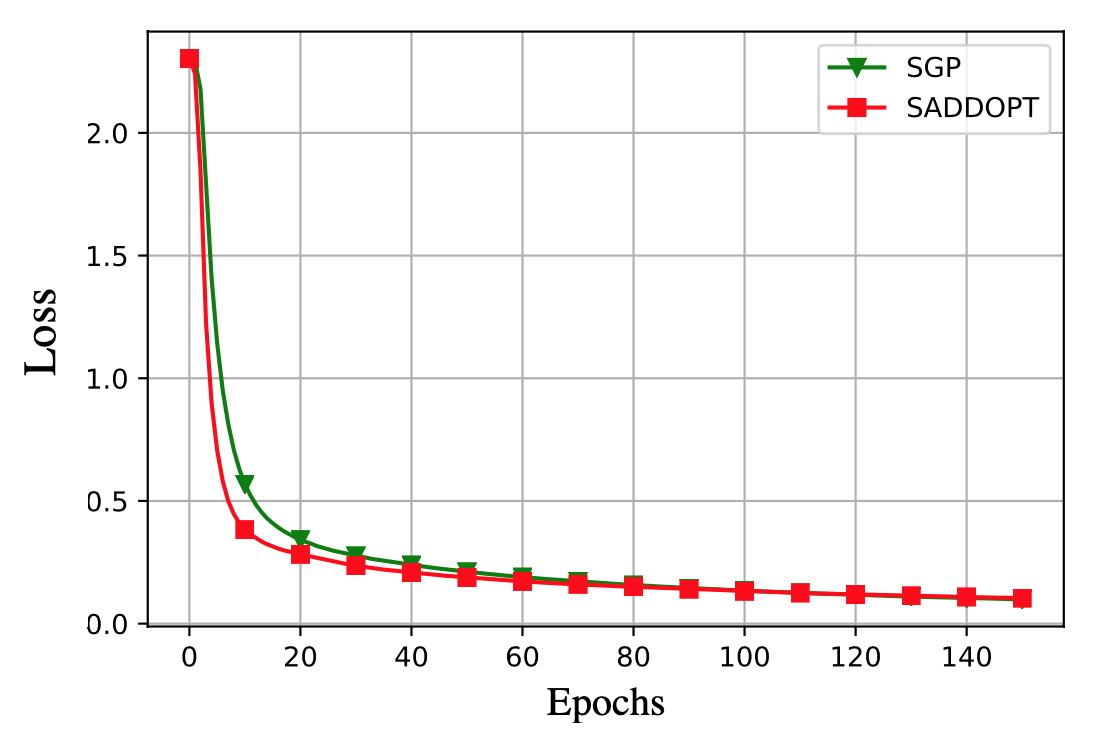}}
\subfigure{\includegraphics[width=2.95in]{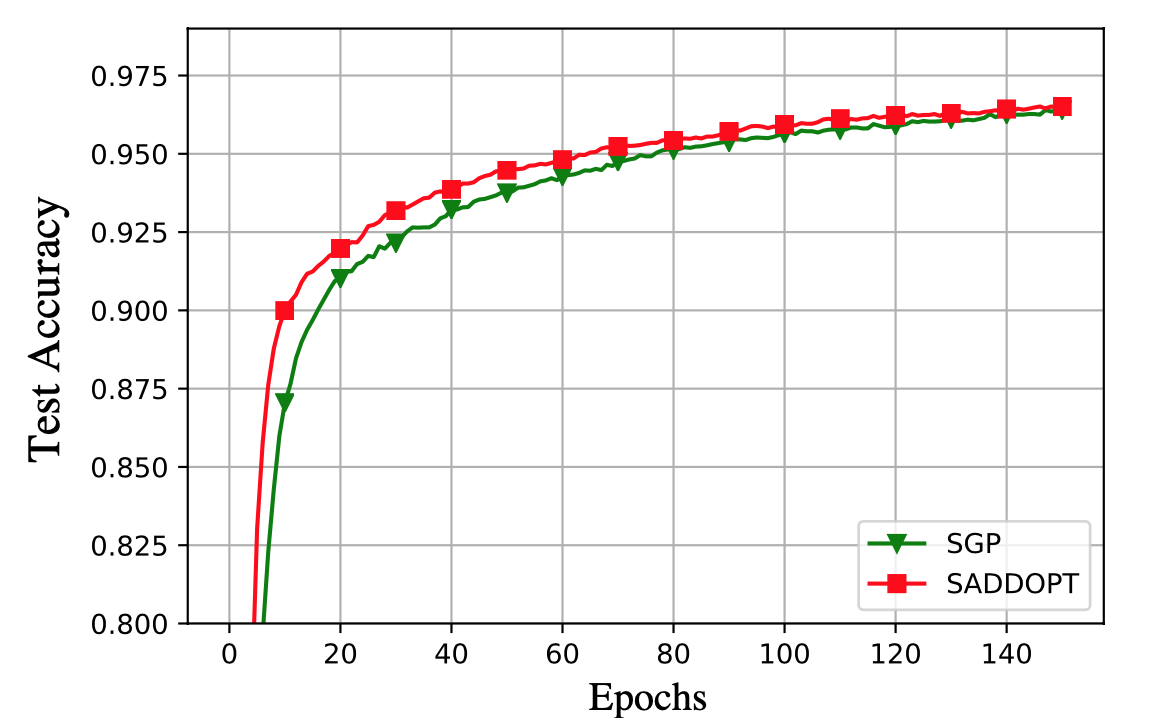}}
\caption{MNIST classification using a two-layer neural network over a directed geometric graph with~$n=500$~nodes and~$m=120$ data samples per node; both algorithms use a constant step-size.}
\label{NN}
\end{figure}

\begin{figure}[!h]
\centering
\subfigure{\includegraphics[width=2.8in]{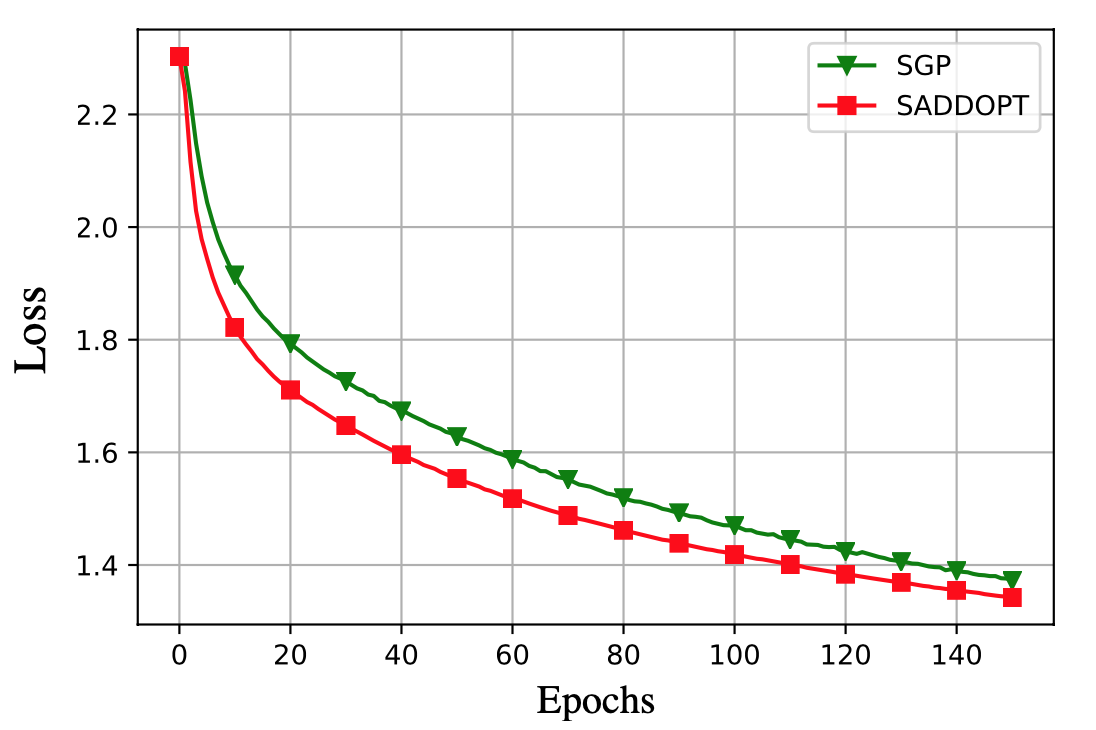}}
\subfigure{\includegraphics[width=2.95in]{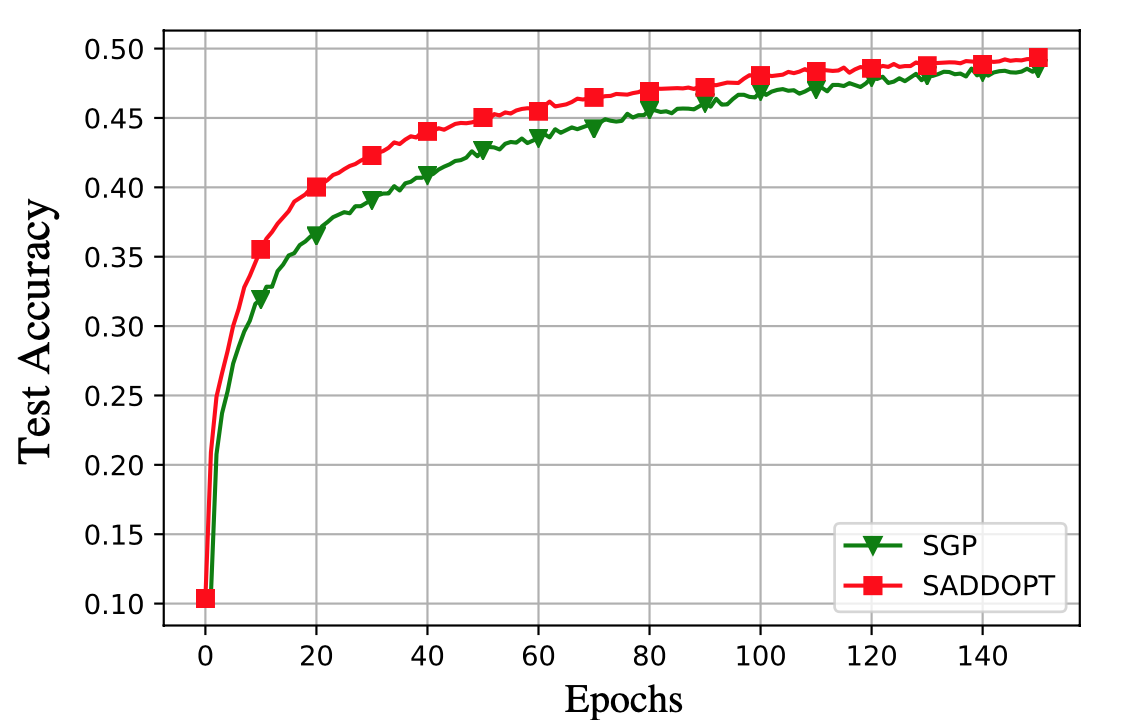}}
\caption{CIFAR-10 classification using a two-layer neural network over a directed geometric graph with~$n=500$~nodes and~$m=120$ data samples per node; both algorithms use a constant step-size.}
\label{NN2}
\end{figure}

\newpage
\section{Conclusions}\label{sec_conc}
In this report, we present~\SA, a decentralized stochastic optimization algorithm that is applicable to both undirected and directed graphs.~\SA~adds gradient tracking to~\SGP~and can be viewed as a stochastic extension of~$\ADDOPT$. We show that for a constant step-size~$\alpha$,~$\SA$ converges linearly inside an error ball around the optimal, the size of which is controlled by~$\alpha$. For decaying step-sizes~$\mc{O}(1/k)$, we show that~$\SA$ is asymptotically network-independent and reaches the exact solution sublinearly at~$\mc{O}(1/k)$. These characteristics match the centralized~\CSGD~up to some constant factors. Numerical experiments over both strongly convex and non-convex problems illustrate the convergence behavior and the  performance comparison of~\SA~ versus~\SGP~and their non-stochastic counterparts.

\bibliographystyle{IEEEbib}
\bibliography{sample,SampleProc,KhanPubs}

\newpage
\appendices

\section{Developing the LTI System Describing~$\SA$}\label{App1}

To derive the LTI system described in~\eqref{sys_conv}, we first define a few terms: 
\begin{align*}
\overline{\mb{w}}_k &:= \frac{1}{n} \mb{1}_n ^\top \mb{w}_k, \quad \overline{\mb{h}}_k := \frac{1}{n} \mb{1}_n ^\top \nabla f (\mb{z}_k), \quad \overline{\mb{g}}_k := \frac{1}{n} \mb{1}_n ^\top \nabla \widehat{f} (\mb{z}_k) := \overline{\mb{w}}_k,\\
\overline{\mb{p}}_k &:= \frac{1}{n} \mb{1}_n^\top \nabla f(\mb{1}_n\ol{\mb{x}}_k), \quad \nabla f(\mb z_k) := [\nabla {f}_1(\mb z^1_k)^\top, \cdots, \nabla {f}_n(\mb z^n_k)^\top]^\top.
\end{align*}
We denote~$\xi_k^i \in \mbb{R}^p$~as random vectors for all~$k \geq 0$~and~$i \in \mc{V}$ such that the stochastic gradient is~$\nabla \widehat{f}_i(\mb z_k^i) = \nabla f_i(\mb z_k^i, \xi_k^i)$. Assumption~\ref{SFO_assump}~allows the gradient noise processes to be dependent on agent~$i$~and the current iterate~$\mb{z}_k^i$. We denote by~$\mc{F}_k$, the~$\sigma$-algebra generated by the set of random vectors~$\{ \xi_l^i \}_{i \in \mc{V}}$, where~$0 \leq l \leq k-1$.
The derivation of the system described in~\eqref{sys_conv} is now provided in the following three steps:

\textbf{Step 1. Network agreement error.}

Note that the first term~$\|\mb{x}_{k+1} - B^\infty \mb{x}_{k+1} \|^2 _{\bds{\pi}}$ in the LTI system is essentially the network agreement error and it can be expanded as:
\begin{align} \label{tk1}
   \|\mb{x}_{k+1} - B^\infty \mb{x}_{k+1} \|^2 _{\bds{\pi}} 
   &= \|B \mb{x}_k - B^\infty \mb{x}_k - \alpha (\mb{w}_k - B^\infty \mb{w}_k) \|^2 _{\bds{\pi}} \nonumber\\
   &= \|B \mb{x}_k - B^\infty \mb{x}_k\|^2 _{\bds{\pi}} + \alpha^2 \|\mb{w}_k - B^\infty \mb{w}_k \|^2 _{\bds{\pi}} - 2 \langle B \mb{x}_k - B^\infty \mb{x}_k, \alpha (\mb{w}_k - B^\infty\mb{w}_k) \rangle_{\bds{\pi}} \nonumber\\
   &\leq \sigma_B^2 \|\mb{x}_k - B^\infty \mb{x}_k\|^2 _{\bds{\pi}} + \alpha^2 \| \mb{w}_k - B^\infty \mb{w}_k \|^2 _{\bds{\pi}} + 2 \alpha \sigma_B \|\mb{x}_k - B^\infty \mb{x}_k\|_{\bds{\pi}} \|\mb{w}_k - B^\infty \mb{w}_k \|_{\bds{\pi}} \nonumber\\
   &\leq \left(\sigma_B^2 + \alpha \sigma_B \frac{1-\sigma_B^2}{2 \alpha \sigma_B} \right) \|\mb{x}_k - B^\infty \mb{x}_k\|^2 _{\bds{\pi}} + \left(\alpha^2 + \alpha \sigma_B \frac{2 \alpha \sigma_B}{1-\sigma_B^2} \right) \| \mb{w}_k - B^\infty \mb{w}_k \|^2 _{\bds{\pi}} \nonumber\\
   &= \left(\frac{1 + \sigma_B^2}{2} \right) \|\mb{x}_k - B^\infty \mb{x}_k\|^2 _{\bds{\pi}} + \alpha^2 \left(\frac{1 + \sigma_B^2}{1-\sigma_B^2} \right) \| \mb{w}_k - B^\infty \mb{w}_k \|^2 _{\bds{\pi}}.
\end{align}

\textbf{Step 2. Optimality gap.}

Next, we consider~$\|\overline{\mb{x}}_{k+1} - \mb{z}^* \|_2 ^2$, which defines the the gap between the mean iterate and the true solution:
\begin{align*}
\|\overline{\mb{x}}_{k+1} - \mb{z}^* \|_2 ^2 
&= \|(\overline{\mb{x}}_{k}-\alpha \overline{\mb{w}}_k) - \mb{z}^* \|_2 ^2 = \|\overline{\mb{x}}_{k}- \mb{z}^*\|_2^2 + \alpha^2 \|\overline{\mb{g}}_k \|_2^2 - 2 \langle \overline{\mb{x}}_{k}- \mb{z}^*, \overline{\mb{g}}_k \rangle.
\end{align*}
Noticing that~$\mbb{E}[\overline{\mb{g}}_k | \mc{F}_k] = \overline{\mb{h}}_k$,
$$\mbb{E}[\|\ol{\mb{g}}_k \|_2^2 | \mc{F}_k] = \mbb{E}[\|\overline{\mb{g}}_k - \overline{\mb{h}}_k\|_2^2 | \mc{F}_k] + \|\overline{\mb{h}}_k\|_2^2 \leq \frac{\sigma^2}{n} + \|\overline{\mb{h}}_k\|_2^2.$$
For~$\eta = (1-\alpha \mu)$, we can write:
\begin{align}
\mbb{E}[\|\overline{\mb{x}}_{k+1} - \mb{z}^* \|_2 ^2 | \mc{F}_k] &\leq \|\overline{\mb{x}}_{k}- \mb{z}^*\|_2^2 - 2 \langle \overline{\mb{x}}_{k}- \mb{z}^*, \overline{\mb{h}}_k \rangle + \alpha^2 \|\overline{\mb{h}}_k \|_2^2 + \frac{\alpha^2 \sigma^2}{n} \nonumber\\
&= \|\overline{\mb{x}}_{k}- \mb{z}^*\|_2^2 - 2 \alpha \langle \overline{\mb{x}}_{k}- \mb{z}^*, \overline{\mb{p}}_k \rangle + 2 \alpha \langle \overline{\mb{x}}_{k}- \mb{z}^*, \overline{\mb{p}}_k - \overline{\mb{h}}_k \rangle + \alpha^2 \|\overline{\mb{p}}_k - \overline{\mb{h}}_k\|_2^2 \nonumber\\
&+ \alpha^2 \|\overline{\mb{p}}_k \|_2^2 - 2 \alpha^2 \langle \overline{\mb{p}}_k, \overline{\mb{p}}_k - \overline{\mb{h}}_k \rangle + \frac{\alpha^2 \sigma^2}{n} \nonumber\\
&= \|\overline{\mb{x}}_{k} - \alpha \overline{\mb{p}}_k - \mb{z}^*\|_2^2 + \alpha^2 \|\overline{\mb{p}}_k - \overline{\mb{h}}_k\|_2^2 + 2 \alpha \langle \overline{\mb{x}}_{k} - \alpha \overline{\mb{p}}_k - \mb{z}^*, \overline{\mb{p}}_k - \overline{\mb{h}}_k \rangle + \frac{\alpha^2 \sigma^2}{n} \nonumber\\
&\leq \eta^2 \|\overline{\mb{x}}_{k} - \mb{z}^*\|_2^2 + \alpha^2 \|\overline{\mb{p}}_k - \overline{\mb{h}}_k\|_2^2 + 2 \alpha \eta \|\overline{\mb{x}}_{k} - \mb{z}^* \|_2 \| \overline{\mb{p}}_k - \overline{\mb{h}}_k \|_2 + \frac{\alpha^2 \sigma^2}{n} \nonumber\\ \label{opt_gap2}
&\leq (1 - \alpha \mu) \|\overline{\mb{x}}_{k} - \mb{z}^*\|_2^2 + \left( \frac{\alpha \ell^2}{n \mu}\right) (1 + \alpha \mu)  \|\mb{1}_n \overline{\mb{x}}_k - \mb{z}_k\|_2^2 + \frac{\alpha^2 \sigma^2}{n}.
\end{align}
It can be verified that~$B^\infty = \frac{1}{n} Y^\infty \mb{1}_n \mb{1}_n^\top$. Next  consider~$\|\mb{z}_k - \mb{1}_n \overline{\mb{x}}_{k} \|_2^2$:
\begin{align*}
\|\mb{z}_k - \mb{1}_n \overline{\mb{x}}_{k} \|_2^2
&= \|Y^{-1}\mb{x}_k - Y^\infty \mb{1}_n \overline{\mb{x}}_k + Y^\infty \mb{1}_n \overline{\mb{x}}_k - \mb{1}_n \overline{\mb{x}}_k \|_2^2 \\
&= \|Y^{-1}(\mb{x}_k - Y^\infty \mb{1}_n \overline{\mb{x}}_k) + (Y^{-1} Y^\infty - I_{n}) \mb{1}_n \overline{\mb{x}}_k\|_2^2\\
&= \|Y^{-1}(\mb{x}_k - B^\infty \mb{x}_k)\|_2^2 + \|(Y^{-1} Y^\infty - I_{n}) \mb{1}_n \overline{\mb{x}}_k\|_2^2 + 2 \langle Y^{-1}(\mb{x}_k - B^\infty \mb{x}_k), (Y^{-1} Y^\infty - I_{n}) \mb{1}_n \overline{\mb{x}}_k \rangle\\
&\leq y_-^2 \|\mb{x}_k - B^\infty \mb{x}_k\|_2^2 + (y_- \beta \sigma_B^k)^2 \|\mb{x}_k\|_2^2 + 2 (y_-) (y_- \beta \sigma_B^k) \|\mb{x}_k - B^\infty \mb{x}_k\|_2 \|\mb{x}_k\|_2\\
&\leq (y_-^2 + y_-^2 \beta \sigma_B) \overline{\pi} \|\mb{x}_k - B^\infty \mb{x}_k\|_{\bds{\pi}}^2 + \left( y_- ^2 \beta^2 \sigma_B^{2k} + y_-^2 \beta \sigma_B^k \right) \|\mb{x}_k\|_2^2.
\end{align*}
Using the above relation in~\eqref{opt_gap2}, we obtain the final expression for~$\mbb{E}\left[\|\overline{\mb{x}}_{k+1} - \mb{z}^* \|^2 _2  |\mc{F}_{k} \right]$.
\begin{align} \label{tk2}
   \mbb{E}\left[\|\overline{\mb{x}}_{k+1} - \mb{z}^* \|^2 _2  |\mc{F}_{k} \right]   
   &\leq (\alpha^2 g_1 + \alpha g_2) \|\mb{x}_{k} - B^\infty \mb{x}_k \|^2 _{\bds{\pi}} + (1 - \alpha \mu) \|\overline{\mb{x}}_{k} - \mb{z}^* \|^2 _2 \nonumber\\
   &+ \alpha^2 \left(\frac{\sigma^2}{n} \right) + (h_1 \sigma_B^k ) \|\mb{x}_{k} \|^2 _2.
\end{align}

{\bf Step 3: Gradient tracking error.}
Finally, we calculate the gradient tracking error~$ \|\mb{w}_{k+1} - B^\infty \mb{w}_{k+1} \|_{\bds{\pi}} ^2$.
\begin{align*}
   \|\mb{w}_{k+1} - B^\infty \mb{w}_{k+1} \|_{\bds{\pi}} ^2 &= \|B \mb{w}_k - B^\infty \mb{w}_{k} + (I_n - B^\infty) (\nabla \widehat{f}(\mb{z}_{k+1}) - \nabla \widehat{f}(\mb{z}_{k})\|_{\bds{\pi}}^2  \\
   &\leq \sigma_B^2 \|\mb{w}_k - B^\infty \mb{w}_{k} \|_{\bds{\pi}}^2  + \mn{I_n - B^\infty}_{\bds{\pi}}^2  \|\nabla \widehat{f}(\mb{z}_{k+1}) - \nabla \widehat{f}(\mb{z}_{k})\|_{\bds{\pi}}^2 \\ 
   &+ 2 \sigma_B \langle \mb{w}_k - B^\infty \mb{w}_{k}, (I_n - B^\infty) (\nabla \widehat{f}(\mb{z}_{k+1}) - \nabla \widehat{f}(\mb{z}_{k}))\rangle_{\bds{\pi}} \\
   &\leq \sigma_B^2  \|\mb{w}_k - B^\infty \mb{w}_{k} \|_{\bds{\pi}}^2  + \|\nabla \widehat{f}(\mb{z}_{k+1}) - \nabla \widehat{f}(\mb{z}_{k})\|_{\bds{\pi}}^2 \\
   &+ 2 \sigma_B \|\mb{w}_k - B^\infty \mb{w}_{k}\|_{\bds{\pi}} \mn{I_n - B^\infty}_{\bds{\pi}} \|\nabla \widehat{f}(\mb{z}_{k+1}) - \nabla \widehat{f}(\mb{z}_{k}))\|_{\bds{\pi}} \\
   &\leq \left(\sigma_B^2 + \sigma_B \frac{1 - \sigma_B^2}{2 \sigma_B} \right) \|\mb{w}_k - B^\infty \mb{w}_{k} \|_{\bds{\pi}}^2  + \left( 1 + \sigma_B \frac{2 \sigma_B}{1 - \sigma_B^2} \right) \|\nabla \widehat{f}(\mb{z}_{k+1}) - \nabla \widehat{f}(\mb{z}_{k})\|_{\bds{\pi}}^2 \\
   &= \left(\frac{1 + \sigma_B^2}{2} \right) \|\mb{w}_k - B^\infty \mb{w}_{k} \|_{\bds{\pi}}^2  + \left(\frac{1 + \sigma_B^2}{1 - \sigma_B^2} \right) \|\nabla \widehat{f}(\mb{z}_{k+1}) - \nabla \widehat{f}(\mb{z}_{k})\|_{\bds{\pi}}^2.
\end{align*}
We bound the second term of the above equation as:
\begin{align*}
\|\nabla \widehat{f}(\mb{z}_{k+1}) - \nabla \widehat{f}(\mb{z}_{k})\|_{\bds{\pi}}^2
&= \|\nabla \widehat{f}(\mb{z}_{k+1}) - \nabla \widehat{f}(\mb{z}_{k}) - (\nabla f(\mb{z}_{k+1}) - \nabla f(\mb{z}_{k})) + \nabla f(\mb{z}_{k+1}) - \nabla f(\mb{z}_{k})\|_{\bds{\pi}}^2 \\
&\leq 2 \ell^2 \ul{\pi}^{-1} \|\mb{z}_{k+1} - \mb{z}_{k}\|_{2}^2 + 2 \| \nabla \widehat{f}(\mb{z}_{k+1}) - \nabla \widehat{f}(\mb{z}_{k}) -(\nabla f(\mb{z}_{k+1}) - \nabla f(\mb{z}_{k}))\|_{\bds{\pi}}^2.
\end{align*}
Consider the first term~$\|\mb{z}_{k+1} - \mb{z}_{k}\|_{2}^2$~of above equation.
\begin{align*}
    \|\mb{z}_{k+1} - \mb{z}_{k}\|_{2}^2
    &= \|Y_{k+1}^{-1}((B \mb{x}_{k} - \alpha \mb{w}_{k}) - \mb{x}_{k}) + (Y_{k+1}^{-1} - Y_{k}^{-1}) \mb{x}_{k} \|_{2}^2\\
    &= \|Y_{k+1}^{-1}(B - I_{n})\mb{x}_{k} - \alpha Y_{k+1}^{-1} \mb{w}_{k} + (Y_{k+1}^{-1} - Y_{k}^{-1}) \mb{x}_{k} \|_{2}^2\\
    &\leq \|Y_{k+1}^{-1}(B - I_{n})\mb{x}_{k}\|_{2}^2 + \alpha^2 \|Y_{k+1}^{-1} \mb{w}_{k}\|_{2}^2 + \|(Y_{k+1}^{-1} - Y_{k}^{-1}) \mb{x}_{k} \|_{2}^2 + 2\|Y_{k+1}^{-1}(B - I_{n})\mb{x}_{k}\|_{2} \|\alpha Y_{k+1}^{-1} \mb{w}_{k}\|_{2} \\ 
    &+ 2 \|\alpha Y_{k+1}^{-1} \mb{w}_{k}\|_{2} \|(Y_{k+1}^{-1} - Y_{k}^{-1}) \mb{x}_{k} \|_{2} + 2\|Y_{k+1}^{-1}(B - I_{n})\mb{x}_{k}\|_{2} \|(Y_{k+1}^{-1} - Y_{k}^{-1}) \mb{x}_{k} \|_{2}\\
    &\leq \|Y_{k+1}^{-1}(B - I_{n})\mb{x}_{k}\|_{2}^2 +  \|\alpha Y_{k+1}^{-1} \mb{w}_{k}\|_{2}^2 + \mn{Y_{k+1}^{-1} - Y_{k}^{-1}}_{2}^2 \|\mb{x}_{k} \|_{2}^2 + 2 \|Y_{k+1}^{-1}(B - I_{n})\mb{x}_{k}\|_{2} \| \alpha Y_{k+1}^{-1} \mb{w}_{k}\|_{2}\\ 
    &+ 2 \alpha \|Y_{k+1}^{-1} \mb{w}_{k}\|_{2} \mn{Y_{k+1}^{-1} - Y_{k}^{-1}}_{2} \|\mb{x}_{k} \|_{2} + 2\|Y_{k+1}^{-1}(B - I_{n})\mb{x}_{k}\|_{2} \mn{Y_{k+1}^{-1} - Y_{k}^{-1}}_{2} \|\mb{x}_{k} \|_{2}\\
    &\leq 12 y_-^2 \ol{\pi} \| \mb{x}_k - B^\infty \mb{x}_k \|^2_{\bds{\pi}} + 3 \alpha^2 y_-^2 \| \mb{w}_k \|_2^2 + 24 y_-^4 \beta^2 \sigma_B^{2k} \|\mb{x}_k\|_2^2.
\end{align*}
Next we bound~$\|\mb{w}_k\|_2^2$,
\begin{align*}
    \|\mb{w}_k\|_2^2 &= \|(\mb{w}_k - Y^\infty \mb{1}_n \overline{\mb{g}}_k) + Y^{-1} Y^\infty \mb{1}_n \overline{\mb{p}}_k + Y^{-1}Y^\infty (\mb{1}_n \overline{\mb{g}}_k - \mb{1}_n \overline{\mb{p}}_k) \|_2^2 \\
    &\leq (2 + r)\|\mb{w}_k - Y^\infty \mb{1}_n \overline{\mb{w}}_k \|_2^2 + 3 \|Y^{-1} Y^\infty \mb{1}_n \overline{\mb{p}}_k\|_2^2 + \left(2 + \frac{1}{r} \right)\|Y^{-1}Y^\infty \mb{1}_n (\overline{\mb{g}}_k - \overline{\mb{p}}_k) \|_2^2 \\
    &\leq (2 + r) \overline{\pi} \|\mb{w}_k - B^\infty \mb{w}_k \|_{\bds{\pi}}^2 + 3 y_-^2 y^2 \ell^2 \|\overline{\mb{x}}_k - \mb{z}^*\|_2^2 + 2 \left(2 + \frac{1}{r} \right) y_-^2 y^2 n \|\overline{\mb{g}}_k - \overline{\mb{h}}_k\|_2^2 \\
    &+ 2 \left(2 + \frac{1}{r} \right) y_-^2 y^2 \ell^2 \|\mb{z}_k - \mb{1}_n \overline{\mb{x}}_k \|_2^2.
\end{align*}
whereas,
\begin{align*}
    \mbb{E}[\| \nabla \widehat{f}(\mb{z}_{k+1}) - \nabla \widehat{f}(\mb{z}_{k}) -(\nabla f(\mb{z}_{k+1}) - \nabla f(\mb{z}_{k}))\|_{\bds{\pi}}^2| \mc{F}_k] = 2 n  \sigma^2 \underline{\mb{\pi}}^{-1}.
\end{align*}
Pick~$r = \frac{k_1}{k_2 \alpha^2} - 2 = \frac{k_1 - 2 k_2 \alpha^2}{ k_2 \alpha^2} > 0 => \frac{1}{r} = \frac{k_2 \alpha^2}{k_1 - 2 k_2 \alpha^2} > 0 $. This will enforce a constraint on~$\alpha$ such that ${\alpha < \sqrt{\frac{k_1}{2 k_2}} = \left(\frac{1-\sigma_B^2}{6 \ell y_-} \right) \sqrt{\frac{\underline{\pi}}{(1+\sigma_B^2) \overline{\pi}}}}$. The term~$\|\mb{z}_k - \mb{1}_n \overline{\mb{x}}_k \|^2_2$ is already simplified in solving for the optimality gap. Putting these in above equation and after taking the expectation, the resultant equation for gradient tracking error becomes:
\begin{align}\label{tk3}
    \mbb{E}\left[\|\mb{w}_{k+1} -B^\infty \mb{w}_{k+1} \|^2 _{\bds{\pi}} |\mc{F}_{k} \right] 
    &\leq (g_3 + \alpha^2 g_{4}) \|\mb{x}_{k} -B^\infty \mb{x}_{k} \|^2 _{\bds{\pi}} + (\alpha^2 g_{5}) \|\overline{\mb{x}}_{k} -\mb{z}^* \|^2 _{2} + C_\sigma  \nonumber\\
    &+ \left(\frac{5+\sigma_B^2}{6} \right) \mbb{E}\left[\|\mb{w}_{k} -B^\infty \mb{w}_{k} \|^2 _{\bds{\pi}} |\mc{F}_{k} \right] + ((h_{2} + \alpha^2 h_{3}) \sigma_B^k) \|\mb{x}_{k}\|^2_{2}. 
\end{align}
Taking full expectation of~\eqref{tk1},~\eqref{tk2},~and~\eqref{tk3}~leads to the system dynamics described by the relation in~\eqref{sys_conv}.

\section{Proof of Corollary~\ref{c1}}\label{App2}

We derive the upper bound on the spectral radius of~$A_\alpha$ under the conditions on step-size described in Theorem~\ref{th1}. Using~\eqref{a1eq} and~\eqref{a2eq}, the characteristic function of~$A_\alpha$ can be calculated as:
\begin{align*}
    \det(\lambda I_3 - A_\alpha) &= (\lambda - a_{11})(\lambda - a_{22})(\lambda - a_{33}) - a_{13}a_{31}(\lambda - a_{22}) - a_{13} a_{21} a_{32}\\
    &\geq (\lambda - a_{11})(\lambda - a_{22})(\lambda - a_{33}) - a_{13}a_{31}(\lambda - a_{22}) - \frac{1}{\Gamma + 1} (1 - a_{22}) [(1 - a_{11})(1 - a_{33}) - a_{13} a_{31}]\\
    &\geq (\lambda - a_{11})(\lambda - a_{22})(\lambda - a_{33}) - \frac{1}{\Gamma}(\lambda - a_{22})(1 - a_{11})(1 - a_{33}) \\
    &- \frac{\Gamma - 1}{\Gamma(\Gamma + 1)} (1 - a_{11})(1 - a_{22})(1 - a_{33}).
\end{align*}
Since the~$\det(\lambda I - A_\alpha)>0$ and the~$\det(\max \{a_{11}, a_{22}, a_{33}\} I - A_\alpha) = \det(a_{22} I - A_\alpha) <0$, the spectral radius~$\rho(A_\alpha) = (a_{22}, 1)$.
Suppose~$\lambda = 1 - \epsilon$ for some~$\epsilon \in (0, \alpha \mu)$, satisfying
\begin{align}
\det(\lambda I_3 - A_\alpha) &\geq \left(1 - \epsilon - \frac{1 +\sigma_B^2}{2} \right)(\alpha \mu - \epsilon)\left(1 - \epsilon - \frac{5 +\sigma_B^2}{6} \right) - \frac{1}{\Gamma}(\alpha \mu - \epsilon)\left(1 - \frac{1+\sigma_B^2}{2} \right)\left(1 - \frac{5+\sigma_B^2}{6} \right) \nonumber\\
 &- \frac{\Gamma - 1}{\Gamma(\Gamma + 1)} \left(1 - \frac{1+\sigma_B^2}{2} \right)(\alpha \mu)\left(1 - \frac{5+\sigma_B^2}{6} \right) \geq 0, \nonumber\\
 \iff &\left(\frac{1-\sigma_B^2 - 2 \epsilon}{2} \right)(\alpha \mu - \epsilon)\left(\frac{1-\sigma_B^2 - 6 \epsilon}{6} \right) - \frac{1}{\Gamma}(\alpha \mu - \epsilon)\left(\frac{1-\sigma_B^2}{2} \right)\left(\frac{1-\sigma_B^2}{6} \right) \nonumber\\
 &- \frac{\Gamma - 1}{\Gamma(\Gamma + 1)} \left(\frac{1-\sigma_B^2}{2} \right)(\alpha \mu)\left(\frac{1-\sigma_B^2}{6} \right) \geq 0, \nonumber\\
\iff &(\alpha \mu - \epsilon) \left[(1-\sigma_B^2 - 2 \epsilon)(1-\sigma_B^2 - 6 \epsilon) - \frac{1}{\Gamma} (1-\sigma_B^2)^2 \right] \geq \frac{\Gamma - 1}{\Gamma(\Gamma + 1)} (1-\sigma_B^2)^2 (\alpha \mu), \nonumber\\ \label{cor1_eps}
\iff &\left(\frac{\alpha \mu - \epsilon}{\alpha \mu}\right) \left[\frac{(1-\sigma_B^2 - 2 \epsilon)(1-\sigma_B^2 - 6 \epsilon)}{(1-\sigma_B^2)^2} - \frac{1}{\Gamma} \right] \geq \frac{\Gamma - 1}{\Gamma(\Gamma + 1)}.
\end{align}
It is sufficient to have
\begin{equation*}
   \epsilon \leq \left( \frac{\Gamma - 1}{\Gamma + 1} \right) \alpha \mu.
\end{equation*}
Notice that,
\begin{equation*}
\left(\frac{\alpha \mu - \epsilon}{\alpha \mu}\right) \geq \left(\frac{\alpha \mu - \left( \frac{\Gamma - 1}{\Gamma + 1} \right) \alpha \mu}{\alpha \mu}\right) = 1 - \left( \frac{\Gamma - 1}{\Gamma + 1} \right) = \frac{\Gamma + 1 - \Gamma + 1}{\Gamma + 1} = \frac{2}{\Gamma + 1}.
\end{equation*}
To verify the upper bound on~$\epsilon$ under the condition on step-size described in Corollary 1,
\begin{equation*}
\epsilon \leq \left( \frac{\Gamma - 1}{\Gamma + 1} \right) \left(\frac{\Gamma + 1}{\Gamma}\right) \left(\frac{1 - \sigma_B^2}{20 \mu}\right) \mu = \left(\frac{\Gamma - 1}{\Gamma}\right) \left(\frac{1 - \sigma_B^2}{20}\right),
\end{equation*}
which implies,
\begin{align*}
   1 - \sigma_B^2 - 2 \epsilon \geq 1 - \sigma_B^2 - 2 \left(\frac{\Gamma - 1}{\Gamma}\right) \left(\frac{1 - \sigma_B^2}{20}\right) &= \frac{(9 \Gamma + 1)(1 - \sigma_B^2)}{10 \Gamma}, \\
   1 - \sigma_B^2 - 6 \epsilon \geq 1 - \sigma_B^2 - 6 \left(\frac{\Gamma - 1}{\Gamma}\right) \left(\frac{1 - \sigma_B^2}{20}\right) &= \frac{(7 \Gamma + 3)(1 - \sigma_B^2)}{10 \Gamma},\\
   \iff (1 - \sigma_B^2 - 2 \epsilon) (1 - \sigma_B^2 - 6 \epsilon) &\geq \frac{(63 \Gamma^2 + 34 \Gamma + 3)(1 - \sigma_B^2)^2}{100 \Gamma^2}.
\end{align*}
Plugging these values in~\eqref{cor1_eps} and for~$\Gamma > 1$, we get,
\begin{align*}
\left(\frac{\alpha \mu - \epsilon}{\alpha \mu}\right) &\left[\frac{(1-\sigma_B^2 - 2 \epsilon)(1-\sigma_B^2 - 6 \epsilon)}{(1-\sigma_B^2)^2} - \frac{1}{\Gamma} \right] \geq \left(\frac{2}{\Gamma + 1} \right) \left[ \frac{\frac{(63 \Gamma^2 + 34 \Gamma + 3)(1 - \sigma_B^2)^2}{100 \Gamma^2} }{(1-\sigma_B^2)^2} - \frac{1}{\Gamma} \right]\\
&= \left(\frac{1}{\Gamma (\Gamma + 1)} \right) \left[\frac{(63 \Gamma^2 + 34 \Gamma + 3)}{50 \Gamma} - 2 \right] = \left(\frac{1}{\Gamma (\Gamma + 1)} \right) \left[\frac{63 \Gamma^2 - 66 \Gamma + 3}{50 \Gamma} \right] \\
&= \left(\frac{1}{\Gamma (\Gamma + 1)} \right) \left[ \Gamma - 1 + \frac{13 \Gamma}{50} - \frac{16}{50} + \frac{3}{50 \Gamma} \right] \geq \frac{\Gamma - 1}{\Gamma (\Gamma + 1)}. 
\end{align*}
Define~$\lambda^* = 1 - \left( \frac{\Gamma - 1}{\Gamma+1} \right) \alpha \mu$. Then the~$\det(\lambda^* I - A_\alpha)\geq0$. Therefore,~$\rho(A_\alpha)\leq \lambda^*$. We select~$\Gamma = 2$~and the corollary follows.

\end{document}